\documentclass{article}

\PassOptionsToPackage{numbers, compress}{natbib}

\usepackage[final]{neurips_2022}

\usepackage[utf8]{inputenc} 
\usepackage[T1]{fontenc}    
\usepackage{hyperref}       
\usepackage{url}            
\usepackage{booktabs}       
\usepackage{amsfonts}       
\usepackage{nicefrac}       
\usepackage{microtype}      
\usepackage{xcolor}         
\usepackage{amsmath,amsthm,amssymb}
\usepackage{braket}
\usepackage{comment}
\usepackage{enumitem}
\usepackage{graphicx}
\usepackage{wrapfig}
\usepackage{algorithm}
\usepackage{nccmath}
\usepackage{algpseudocode}
\usepackage{soul}
\usepackage{xcolor}

\DeclareMathOperator{\E}{\mathbb{E}}

\theoremstyle{plain}
\newtheorem{theorem}{Theorem}
\newtheorem{lemma}{Lemma}
\newtheorem{proposition}{Proposition}

\theoremstyle{definition}

\theoremstyle{remark}

\newtheorem{remark}{Remark}

\newcommand{\red}[1]{\textcolor{red}{#1}} 
\newcommand{\blue}[1]{\textcolor{blue}{#1}}

\title{Minimax Optimal Fixed-Budget Best Arm Identification in Linear Bandits}

\author{%
    Junwen Yang 
    \\
  Institute of Operations Research and Analytics\\
  National University of Singapore\\
  \texttt{junwen\_yang@u.nus.edu} \\
  \And
  Vincent Y. F. Tan \\
   Department of Mathematics \\
   Department of Electrical and Computer Engineering\\
     Institute of Operations Research and Analytics\\
    National University of Singapore \\
  \texttt{vtan@nus.edu.sg} \\
}

\begin{document}

\maketitle

\begin{abstract}
    We study the problem of best arm identification in linear bandits in the fixed-budget setting. By leveraging properties of the G-optimal design and incorporating it into the arm allocation rule, we design a parameter-free algorithm, Optimal Design-based Linear Best Arm Identification (OD-LinBAI). We provide a theoretical analysis of the failure probability of OD-LinBAI. Instead of all the optimality gaps, the performance of OD-LinBAI depends only on the gaps of the top $d$ arms, where $d$ is the effective dimension of the linear bandit instance. Complementarily, we present a minimax lower bound for this problem. The upper and lower bounds show that OD-LinBAI is minimax optimal up to constant multiplicative factors in the exponent, which is a significant theoretical improvement over existing methods (e.g., BayesGap, Peace, LinearExploration and GSE), and settles the question of ascertaining the difficulty of learning the best arm in the fixed-budget setting. Finally, numerical experiments  demonstrate considerable empirical improvements over existing algorithms on a variety of real and synthetic datasets.
\end{abstract}

\section{Introduction}
The multi-armed bandit problem is a model that exemplifies the exploration-exploitation tradeoff in online decision making. It    has various applications in drug design, online advertising, recommender systems, and so on. In stochastic multi-armed bandit problems, the agent sequentially chooses an arm from the given arm set at each time step and then observes a random reward drawn from the unknown distribution associated with the chosen arm. 

The standard multi-armed bandit problem, where the arms are not correlated with one another, has been studied extensively in the literature. While the \emph{regret minimization} problem aims at maximizing the cumulative rewards by the trade-off between exploration and exploitation \citep{thompson1933likelihood, auer2002finite, bubeck2012regret,agrawal2012analysis}, the \emph{pure exploration} problem focuses on efficient exploration with specific goals, e.g., to identify the best arm \citep{even2006action,bubeck2009pure,audibert2010best,gabillon2012best,karnin2013almost,carpentier2016tight,garivier2016optimal}. There are two complementary settings for the problem of \emph{best arm identification}: (i) Given $T \in \mathbb N$, the agent aims to maximize the probability of finding the best arm in at most $T$ time steps; (ii) Given $\delta > 0$, the agent aims to find the best arm with the probability of at least $1-\delta$ in the smallest number of steps. These settings are respectively known as the fixed-budget and fixed-confidence settings.

In this paper, we consider the problem of best arm identification in linear bandits in the fixed-budget setting. In linear bandits, the arms are correlated through an unknown global regression parameter vector $\theta^{*} \in \mathbb{R}^{d}$. In particular, each arm $i$ from the arm set $\mathcal A$ is associated with an arm vector $a(i) \in \mathbb{R}^{d}$, and the expected reward of arm $i$ is given by the inner product between $\theta^{*}$ and $a(i)$. Hence, the standard multi-armed bandits and linear bandits are fundamentally different due to the fact that for the latter, pulling one arm can indirectly reveal information about the other arms but in the former, the arms are independent.  

A wide range of applications in practice can be modeled by linear bandits. For example, \citet{tao2018best} considered online advertising, where the goal is to select an advertisement from a pool to maximize the probability of clicking for web users with different features. Empirically, the probability of clicking can be approximated by a linear combination of various attributes associated with the user and the advertisements (such as age, gender, the domain, keywords, advertising genres, etc.). Moreover, \citet{hoffman2014correlation} applied the linear bandit model into the traffic sensor network problem and the problem of automatic model selection and algorithm conﬁguration.

\noindent\textbf{Main contributions.} Our main contributions are as follows:
\begin{enumerate}[label = (\roman*)]
    \item  We design an algorithm    {\em Optimal Design-based Linear Best Arm Identification} (OD-LinBAI). This {\em computationally efficient} algorithm utilizes a phased elimination-based strategy in which the number of times each arm is pulled in each phase depends on G-optimal designs \citep{kiefer1960equivalence}. Besides, OD-LinBAI is {\em totally parameter-free}, whereas some existing methods (e.g., BayesGap and Peace) require the knowledge of the problem instance (which is typically not known in practice).
    
    \item We derive an upper bound on the failure probability of OD-LinBAI. The failure probability is a significant improvement over those of existing methods which we survey in detail in Section~\ref{sec:comparisons}.  In particular, we show that the exponent of the failure probability depends on a hardness quantity $H_{2,\mathrm{lin}}$. This quantity is a function of {\em only} the first $d-1$ optimality gaps, where $d$ is the dimension of the arm vectors. This is a surprising and significant difference compared to the upper bounds of the failure probabilities of various algorithms for best arm identification in standard multi-armed bandits \citep{audibert2010best,karnin2013almost,shahrampour2017sequential} and BayesGap \citep{hoffman2014correlation} in linear bandits, which all depend on a hardness quantity that depends on {\em all} the gaps. Moreover, OD-LinBAI improves the exponent of the error probability by a factor of $\Theta(\log d)$ over Peace \citep{katz2020empirical} in the worst-case sense or a factor of $\Theta((\log K)/(\log d))$ (which could be much larger than~$1$) over LinearExploration \citep{alieva2021robust} and GSE \citep{azizi2021fixed} in general.

    \item Lastly, using ideas from~\citet{carpentier2016tight}, we prove a minimax lower bound which involves another hardness quantity $H_{1,\mathrm{lin}}$. By comparing $H_{1,\mathrm{lin}}$ to $H_{2,\mathrm{lin}}$, we show that OD-LinBAI is minimax optimal up to constants in the exponent. OD-LinBAI is the first algorithm that provably achieves minimax optimality in this problem, and finally settles the question of ascertaining the hardness of learning the best arm in the fixed-budget setting for linear bandits. In addition, experiments in both synthetic and real-world datasets firmly corroborate the efficacy of OD-LinBAI vis-\`a-vis other existing methods. 
\end{enumerate}
  
\noindent\textbf{Related work.}  The problem of regret minimization in linear bandits was first studied by \citet{abe1999associative}, and has attracted extensive interest in the development of various algorithms (e.g., UCB-style algorithms \citep{auer2002using,dani2008stochastic,rusmevichientong2010linearly,abbasi2011improved,chu2011contextual}, Thompson sampling \citep{agrawal2013thompson, abeille2017linear}). In particular, in the book of \citet{lattimore2020bandit}, a regret minimization algorithm based on the G-optimal design was proposed for linear bandits with finitely many arms. Although both this algorithm and our algorithm OD-LinBAI utilize the G-optimal design technique, they differ in numerous aspects including the manner of elimination and arm allocation, which emanates from the two different objectives.

For the problem of best arm identification in linear bandits, the fixed-confidence setting has previously been studied in \citep{ tao2018best, soare2014best, xu2018fully,zaki2019towards,fiez2019sequential,degenne2020gamification,jedra2020optimal, kazerouni2021best}. In particular, \citet{soare2014best} introduced the optimal G-allocation problem and proposed a static algorithm $\mathcal X \mathcal Y$-Oracle as well as a semi-adaptive algorithm $\mathcal X \mathcal Y$-Adaptive; see Remark~\ref{rmk:g-opt} for more discussions on \citet{soare2014best}.  \citet{degenne2020gamification} treated the problem as a two-player zero-sum game between the agent and the
nature, and thus designed an asymptotically optimal algorithm for the fixed-confidence setting.

The fixed-budget setting for the problem of best arm identification in linear bandits has also been studied in a few previous and concurrent works. \citet{hoffman2014correlation} introduced a gap-based exploration algorithm BayesGap, which is a Bayesian treatment of UGapEb \citep{gabillon2012best} for standard multi-armed bandits. Peace by \citet{katz2020empirical} utilizes an experimental design based on the Gaussian-Width of the underlying arm set, which characterizes the geometry of the instance better in some instances. However, both BayesGap and Peace are computationally expensive and not parameter-free. Recently, \citet{alieva2021robust} introduced an elimination algorithm named LinearExploration, which is also robust to moderate levels of model misspecification. Generalized Successive Elimination (GSE) by \citet{azizi2021fixed} shares a similar structure with LinearExploration and applies to generalized linear models. Nevertheless, none of the above is minimax optimal. See Section~\ref{section_mainresults} and Section~\ref{section_exp} for more comparisons between OD-LinBAI and other existing algorithms.

\section{Problem setup and preliminaries}
\label{section_setup}
\noindent\textbf{Best arm identification in linear bandits.} We consider the standard linear bandit problem with an unknown global regression parameter. In a linear bandit instance $\nu$, the agent is given an arm set $\mathcal A = [K]$, which corresponds to known arm vectors $\{a(1), a(2),\ldots,a(K)\} \subset \mathbb R^d$. At each time $t$, the agent chooses an arm $A_t$ from the arm set $\mathcal A$ and then observes a noisy reward 
$$
X_{t}=\braket{\theta^{*}, a(A_{t})}+\eta_{t}
$$
where $\theta^{*} \in \mathbb{R}^{d}$ is the unknown parameter vector and $\eta_{t}$ is independent zero-mean $1$-subgaussian random noise.

In the fixed-budget setting, given a time budget $T \in \mathbb N$, the agent aims at maximizing the probability of identifying the best arm, i.e., the arm with the largest expected reward, with no more than $T$ arm pulls. More formally, the agent uses an \emph{online} algorithm $\Pi$ to decide the arm $A_t^{\Pi}$ to pull at each time step $t$, and the arm $i_{\mathrm{out}}^{\Pi} \in \mathcal A$  to output as the identified best arm by time $T$. We abbreviate $A_t^{\Pi}$ as $A_t$ and $i_{\mathrm{out}}^{\Pi}$ as $i_{\mathrm{out}}$ when there is no ambiguity.

For any arm $i\in \mathcal A$, let $p(i) = \braket{\theta^{*}, a(i)}$ denote the expected reward. For convenience, we assume that the expected rewards of the arms are in descending order and the best arm is unique. That is to say, $p(1) > p(2) \ge \dots \ge p(K)$. For any suboptimal arm $i$, we denote $\Delta_i = p(1) - p(i)$ as the optimality gap. For ease of notation, we also set $\Delta_1= \Delta_2$. Furthermore, let $\mathcal E$ denote the set of all the linear bandit instances defined above.

\noindent\textbf{Dimensionality-reduced arm vectors.} For any linear bandit instance, if the corresponding arm vectors do not span $\mathbb R ^d$, i.e., $\operatorname{span} ( \{a(1), a(2),\dots,a(K)\} )  \subsetneq \mathbb R ^d$, the agent can work with a set of dimensionality-reduced arm vectors $\{a'(1), a'(2),\dots,a'(K)\} \subset \mathbb R^{d'}$, that spans $\mathbb R ^{d'}$, with little consequence. Specifically, let $B \in \mathbb R ^{d\times d'} $ be a matrix whose columns form an orthonormal basis of the subspace spanned by $a(1), a(2),\dots,a(K)$.\footnote{Such an orthonormal basis can be calculated efficiently with the reduced singular value decomposition, Gram–Schmidt process, etc.} Then the agent can simply set $a'(i) = B^{\top}a(i)$ for each arm $i$. To verify this, notice that $BB^{\top}$ is a projection matrix onto the subspace spanned by $ \{a(1), a(2),\dots,a(K)\} $ and consequently
\begin{align*}
    p(i) &= \braket{\theta^*, a(i)} = \braket{\theta^*, BB^{\top}a(i)} 
    = \braket{B^{\top}\theta^*, B^{\top}a(i)} = \braket{{\theta^*}', a'(i)} .
\end{align*}Note that $\theta^*$ is the unknown parameter vector for original arm vectors while ${\theta^*}' = B^{\top}\theta^*$ is the corresponding unknown parameter vector for the dimensionality-reduced arm vectors. In the problem of linear bandits, what we really care about is not the original unknown parameter $\theta^*$ itself but the inner products between $\theta^*$ and the arm vectors $a(i)$, which establishes the equivalence of original arm vectors and dimensionality-reduced arm vectors. 

In our work, without loss of generality, we assume that the entire set of original arm vectors $ \{a(1), a(2),\dots,a(K)\} $ span $\mathbb R ^d$ and $d\ge 2$.\footnote{The situation that $d = 1$ is trivial: each arm vector is a scalar multiple of one another.} However, this idea of transforming into dimensionality-reduced arm vectors is often used in our elimination-based algorithm. See Section~\ref{section_algo} for details.

\noindent\textbf{Least squares estimators.} Let $A_1, A_2, \dots, A_n$ be the sequence of arms pulled by the agent and $X_1,X_2,\dots,X_n$ be the corresponding noisy rewards. Suppose that the corresponding arm vectors $\{a(A_1), a(A_2), \dots, a(A_n)\}$ span $\mathbb R ^d$, then the ordinary least squares (OLS) estimator of $\theta^*$ is given by
\begin{equation*}
\hat{\theta}=V^{-1} \sum_{t=1}^{n} a(A_{t}) X_{t}
\end{equation*}
where $V=\sum_{t=1}^{n}  a(A_t) a(A_t)^{\top} \in \mathbb R^{d\times d}$ is invertible. By applying the properties of subgaussian random variables, a confidence bound for the OLS estimator can be derived as follows. 
\begin{proposition}[{\citet[Chapter 20]{lattimore2020bandit}}]
\label{prop_concentration_estimator}
If $A_1, A_2, \dots, A_n$ are deterministically chosen without knowing the realizations  of $X_1,X_2,\dots,X_n$, then for any $a \in \mathbb R^d$ and $\delta > 0$, 
\begin{equation*}
    \Pr \left[ \braket{\hat{\theta}-\theta^*, a}\ge \sqrt{2\|a\|^2_{V^{-1}} \log\left(\frac 1 {\delta} \right)}\right] \le \delta.
\end{equation*}
\end{proposition}

\begin{remark} 
When the arm pulls are adaptively chosen according to the random rewards, Proposition~\ref{prop_concentration_estimator} no longer applies and an extra factor $\sqrt d$ has to be paid for adaptive arm pulls \citep{abbasi2011improved}. Our algorithm avoids this issue by deciding the arm pulls at the beginning of each phase, and designing the OLS estimator only based on the information from the current phase. See Section~\ref{section_algo} for details.
\end{remark}

\noindent\textbf{G-optimal design.} The confidence interval in Proposition~\ref{prop_concentration_estimator} shows the strong connection between the arm allocation in linear bandits and experimental design theory \citep{pukelsheim2006optimal}. To control the confidence bounds, we first introduce the G-optimal design technique into the problem of best arm identification in linear bandits in the fixed-budget setting. Formally, the G-optimal design problem aims at finding a probability distribution $\pi: \{a(i):i\in \mathcal A \} \rightarrow[0,1]$ that minimises 
\begin{equation*}
    g(\pi) = \max_{i\in \mathcal A} \| a(i) \|^2_{V(\pi)^{-1}}
\end{equation*}
where $V(\pi) = \sum_{i\in \mathcal A} \pi(a(i))a(i)a(i)^{\top}$. Theorem~\ref{theorem_goptimal} states the existence of a small-support G-optimal design and the minimum value of $g$.

\begin{theorem}[\citet{kiefer1960equivalence}]
\label{theorem_goptimal}
If the arm vectors $\{a(i):i\in \mathcal A \}$ span $\mathbb R ^d$, the following statements are equivalent: (i) $\pi^*$ is a minimiser of $g$; (ii) $\pi^*$  is a maximiser of $f(\pi) = \log \det V(\pi)$; (iii) $g(\pi^*) = d$.
Furthermore, there exists a minimiser $\pi^*$ of $g$ such that $|\operatorname{Supp}\left(\pi^*\right)| \leq d(d+1) / 2$. 
\end{theorem}

\begin{remark} \label{rmk:g-opt}
It is worth mentioning that the G-optimal design problem for finite arm vectors is a convex optimization problem while the original G-allocation problem in \citet{soare2014best} for the fixed-confidence best arm identification in linear bandits is an NP-hard discrete optimization problem. A classical algorithm to solve the G-optimal design problem is the Frank--Wolfe algorithm \citep{frank1956algorithm}, whose modified version guarantees linear convergence \citep{damla2008linear}. For our work, it is sufficient to compute an $\epsilon$-approximate optimal design\footnote{For an $\epsilon$-approximate optimal design $\pi$, $g(\pi) \le (1+\epsilon)d$.} with minimal impact on performance. Recently, a near-optimal design with smaller support was proposed in \citet{lattimore2020learning}, which might be helpful in some scenarios. See Appendix~\ref{appendix_goptimal} for more discussions on the above issues. To reduce clutter and ease the reading, henceforward in the main text, we assume that a G-optimal design for finite arm vectors can be found accurately and efficiently.
\end{remark}

\section{Algorithm}
\label{section_algo}
Pseudocode for our algorithm {\em  Optimal Design-based Linear Best Arm Identification} (OD-LinBAI) is presented in Algorithm~\ref{algo1}.

\begin{algorithm}[!t]
\caption{Optimal Design-based Linear Best Arm Identification (OD-LinBAI)} 
\label{algo1}
\hspace*{0.00in} {\bf Input:} time budget $T$, arm set $\mathcal A = [K]$ and arm vectors $\{a(1), a(2),\dots,a(K)\} \subset \mathbb R^d.$
\begin{algorithmic}[1]
\State Initialize $t_0 = 1$, $ \mathcal A_{0} \leftarrow \mathcal A$ and $d_0=d$.
\State For each arm $i \in \mathcal A_{0}$, set $a_0(i) = a(i)$.
\State Calculate $m$ using Equation~(\ref{equation_m}).
\For{$r=1$ to $\lceil\log _{2} d\rceil$} 
\State Set $d_r = \dim \left (\operatorname{span}\left(\{a_{r-1}(i):i\in \mathcal A_{r-1}\}\right)\right)$.
\If{$d_r = d_{r-1}$}
\State For each arm $i \in \mathcal A_{r-1}$, set $a_r(i) = a_{r-1}(i)$.
\Else
\State Find matrix $B_r \in \mathbb R^{d_{r-1}\times d_r}$ whose columns form a orthonormal basis of the subspace spanned by $\{a_{r-1}(i):i\in \mathcal A_{r-1}\}$.
\State For each arm $i \in \mathcal A_{r-1}$, set $a_r(i) = B_r^{\top}a_{r-1}(i)$.
\EndIf
\If{$r=1$}
\State Find a G-optimal design $\pi_r:\{a_r(i):i\in \mathcal A_{r-1}\}\rightarrow[0,1]$ with $|\operatorname{Supp}\left(\pi_{r}\right)| \leq  \frac{d(d+1)}{2} $. 
\Else
\State Find a G-optimal design $\pi_r:\{a_r(i):i\in \mathcal A_{r-1}\}\rightarrow[0,1]$.
\EndIf
\State Set 
$$
T_r(i) = \left\lceil \pi_r(a_r(i)) \cdot m \right\rceil \;\;\text { and } \;\; 
T_r = \sum_{i \in  \mathcal A_{r-1}} T_r(i).
$$
\State Choose each arm $i \in  \mathcal A_{r-1}$ exactly $T_r(i)$ times.
\State Calculate the OLS estimator: 

$$
\hat{\theta}_{r}=V_{r}^{-1} \sum_{t=t_{r}}^{t_{r}+T_{r}-1} a_r(A_{t}) X_{t}\quad
\text{with}
\quad V_{r}=\sum_{i \in \mathcal{A}_{r-1}} T_{r}(i) a_r(i) a_r(i)^{\top}.
$$
\State  For each arm $i \in  \mathcal A_{r-1}$, estimate the expected reward:
$$
\hat{p}_{r}(i) = \braket{\hat{\theta}_{r},a_r(i)}.
$$
\State Let $ \mathcal A_{r}$ be the set of $\lceil  d / 2^r\rceil$ arms in $ \mathcal A_{r-1}$ with the largest estimates of the expected rewards.
\State Set $t_{r+1} = t_r+T_r$.
\EndFor
\end{algorithmic}
\hspace*{0.00in} {\bf Output:} the only arm $i_{\mathrm{out}}$ in $ \mathcal A_{\lceil\log _{2} d \rceil}$.
\end{algorithm}
The algorithm partitions the whole horizon into $\lceil\log _{2} d \rceil$ phases, and maintains an \emph{active} arm set $\mathcal A_r$ in each phase $r$. The length of each phase roughly equals $m$, which will be formally defined in~(\ref{equation_m}). 

Motivated by the equivalence of the original arm vectors and the dimensionality-reduced arm vectors, at the beginning of each phase $r$, the algorithm computes a set of dimensionality-reduced arm vectors $\{a_r(i):i\in \mathcal A_{r-1}\} \subset \mathbb R ^{d_r}$ which spans the $d_r$-dimensional Euclidean space $\mathbb R ^{d_r}$. This can be implemented based on the dimensionality-reduced arm vectors of the last phase $\{a_{r-1}(i):i\in \mathcal A_{r-1}\} $ in an iterative manner (Lines $5-11$).

After that, Algorithm~\ref{algo1} finds a G-optimal design $\pi_r$ for the current dimensionality-reduced arm vectors, with a restriction on the cardinality of the support when $r=1$. OD-LinBAI then pulls each arm in $\mathcal A _{r-1}$ according to the proportions specified by the optimal design $\pi_r$. Specifically, the algorithm chooses each arm $i \in  \mathcal A_{r-1}$ exactly $T_r(i) = \left\lceil \pi_r(a_r(i)) \cdot m \right\rceil$ times, where the parameter $m$ is fixed among different phases and defined as 
\begin{equation}
\label{equation_m}
    m = \frac {T-\min(K, \frac { d (d +1)} 2) -\sum\limits_{r=1}^{{\lceil\log _{2}  d \rceil}-1} {\left\lceil\frac {d} {2^{r}}\right\rceil}} {{\lceil\log _{2}  d \rceil} }.
\end{equation}
Note that $m = \Theta(T/\log_2 d)$ as $T\rightarrow \infty$ with $K$ fixed. Lemma~\ref{lemma_m} in Appendix~\ref{appendix_upperbound} shows with such choice of $m$, the total time budget consumed by the agent is no more than $T$. The parameter $m$ plays a significant role in the implementation as well as the theoretical analysis of Algorithm~\ref{algo1}.

Since the support of the G-optimal design $\pi_r$ must span $\mathbb R ^{d_r}$, the OLS estimator can be directly applied (Line $19$). Then for each arm $i \in \mathcal A _{r-1}$, an estimate of the expected reward is derived. Algorithm~\ref{algo1} decouples the estimates of different phases and only utilizes the information obtained in the current phase $r$. 

At the end of each phase $r$, Algorithm~\ref{algo1} eliminates a subset of possibly suboptimal arms. In particular, $K-\lceil d / 2\rceil$ arms are eliminated in the first phase and about half of the active arms are eliminated in each of the following phases. Eventually, there is only single arm $i_{\mathrm{out}}$ in the active set, which is the output of Algorithm~\ref{algo1}.

\begin{remark} 
\label{rmk:standard}
It is worth considering the case of standard multi-armed bandits, which can be modeled as a special case of linear bandits. In particular, for any arm $i\in \mathcal A=[K]$, the corresponding arm vector is chosen to be $e_i$, which is the $i^{\mathrm{th}}$ standard basis vector of $\mathbb{R}^K$. It follows that $ d = K$, $\theta ^ * = [p(1),p(2),\ldots,p(K)]^{\top} \in \mathbb{R}^K$ and arms are not correlated with one another. A simple mathematical derivation shows that we can always use a set of standard basis vectors of $\mathbb{R}^{d_r}$ to represent the arm vectors regardless of which arms remain active during phase $r$. Also, the G-optimal design for a set of standard basis vectors is the uniform distribution on all of the active arms. Since pulling one arm does not provide information about the other arms, the empirical estimates based on the OLS estimator are exactly the empirical means. Altogether, for standard multi-armed bandits, OD-LinBAI reduces to the procedure of Sequential Halving \citep{karnin2013almost}, which is a state-of-the-art algorithm for best arm identification in standard multi-armed bandits in the fixed-budget setting.
\end{remark}

\begin{remark}
 The G-optimal design steps in Lines $13$ and $15$ in OD-LinBAI may be replaced by the $\mathcal{XY}$-allocation \citep{soare2014best} or other techniques in experimental designs. However, our work focuses on establishing minimax optimality and thus the application of G-optimal designs, which optimize over the worst cases, is natural. The $\mathcal{XY}$-allocation may result in better empirical performance but the improvement might be limited or even absent in worst-case scenarios. More importantly, as noted in \citet[Remark 1]{degenne2020gamification}, for the general $\mathcal{XY}$-allocation problem, only heuristic solutions can be obtained (without convergence guarantees). Nevertheless, the G-optimal design problem can be provably solved with a linear convergence guarantee \citep{damla2008linear}. 
 Overall, the implementation of OD-LinBAI is computationally very efficient. 
\end{remark}

\section{Main results}
\label{section_mainresults}
\subsection{Upper bound}
We first state an upper bound on the error probability of OD-LinBAI (Algorithm~\ref{algo1}). The proof of Theorem~\ref{theorem_upperbound} is deferred to Appendix~\ref{appendix_upperbound}.

\begin{theorem}
\label{theorem_upperbound}
For any linear bandit instance $\nu \in \mathcal E$, OD-LinBAI outputs an arm $i_{\mathrm{out}}$ satisfying 
$$
\Pr\left[i_{\mathrm{out}} \neq 1 \right] \le
\left( \frac {4K} { d}+3\log _{2} d\right) \exp \left( - \frac {m} {32H_{2,\mathrm{lin}}} \right)
$$where  $m$ is defined in Equation~(\ref{equation_m}) and
$$
H_{2,\mathrm{lin}} = \max_{2\leq i \leq d} \frac {i} {\Delta_i^2}.
$$
\end{theorem}

Theorem~\ref{theorem_upperbound} shows the error probability of OD-LinBAI is upper bounded by  \begin{equation}
\label{equation_upper_OD}
    \exp\left(-\Omega\left( \frac{T}{H_{2,\mathrm{lin}} \log_2d }\right) \right)
\end{equation}which depends on $T$, $d$ and $H_{2,\mathrm{lin}}$. We remark that none of the three terms is avoidable in view of our lower bounds (see Section~\ref{subsection_lower}).

In particular, $T$ is the time budget of the problem and $d$ is the effective dimension of the arm vectors.\footnote{Recall that we assume the entire set of original arm vectors $ \{a(1), a(2),\dots,a(K)\} $ span $\mathbb R ^d$.} Given $T$ and $d$, $H_{2,\mathrm{lin}}$ quantifies the difficulty of identifying the best arm in the linear bandit instance. The parameter $H_{2,\mathrm{lin}}$ generalizes its analogue
$$
H_{2} = \max_{2\leq i \leq K} \frac {i} {\Delta_i^2}
$$proposed by \citet{audibert2010best} for standard multi-armed bandits. However, $H_{2,\mathrm{lin}}$ is not larger than $H_2$ since $H_{2,\mathrm{lin}}$ is only a function of the first $d-1$ optimality gaps while $H_2$ considers all of the $K-1$ optimality gaps. In the extreme case that all of the suboptimal arms have the same optimality gaps, i.e., $\Delta_2=\Delta_3=\cdots=\Delta_K$, the two terms $H_2$ and $H_{2,\mathrm{lin}}$ can differ significantly. In general, we have
\begin{equation*}
    H_{2,\mathrm{lin}}\le H_{2} \le \frac K d H_{2,\mathrm{lin}}
\end{equation*}and both inequalities are essentially sharp, i.e., can be achieved by some linear bandit instances. This highlights a major difference between best arm identification in the fixed-budget setting for linear bandits and standard multi-armed bandits. Due to the linear structure, arms are correlated and we can estimate the mean reward of one arm with the help of the other arms. Thus, the hardness quantity $H_{2,\mathrm{lin}}$ is  only a function of the top $d$ arms rather than all the arms.

\subsection{Comparisons to other algorithms} \label{sec:comparisons}
We compare OD-LinBAI and other existing algorithms with respect to the algorithm design as well as the theoretical guarantees in the following.

\noindent\textbf{Comparisons to BayesGap \citep{hoffman2014correlation}.} 
\begin{enumerate}[label = (\roman*), noitemsep, nolistsep]
\item The model used in BayesGap \citep{hoffman2014correlation} is based on Bayesian linear bandits, where the unknown parameter vector $\theta^*$ is drawn from a known prior distribution $\mathcal{N}(0, \eta^{2} I)$ and the additive noise is required to be Gaussian. However, OD-LinBAI does not require these assumptions and the upper bound holds for any deterministic or random $\theta^* \in \mathbb R^d$.
\item The algorithm and theoretical guarantee of BayesGap explicitly require the knowledge of a hardness quantity 
$H_1 = \sum _{1\le i\le K} \Delta_i^{-2}$
to control the confidence region and then allocate exploration. However, this hardness quantity $H_1$ is almost always unknown to the agent in practice. In most practical applications, BayesGap has to estimate $H_1$ in an adaptive way, which works reasonably well in numerical experiments but lacks theoretical guarantees. 

\item  BayesGap's error probability  is upper bounded by
\begin{equation}
\label{equation_upper_BayesGap}
    \exp\left(-\Omega\left( \frac{T}{H_{1}}\right) \right)
\end{equation}which depends on $T$ and $H_1$. Compared with (\ref{equation_upper_BayesGap}), the upper bound of OD-LinBAI in (\ref{equation_upper_OD}) has an extra $\log_2 d$ term. 
This is an interesting phenomena which is also present in standard multi-armed bandits \citep{audibert2010best, carpentier2016tight}. For best arm identification in standard multi-armed bandits, without the knowledge of the hardness quantity $H_1$, the agent has to pay a price of $\log_2 K$ for the adaptation to the problem complexity. In Theorem~\ref{theorem_lowerbound}, we prove a similar result for linear bandits, in which the price of adaptation is $\log_2 d$.

The upper bound (\ref{equation_upper_BayesGap}) involves $H_1$,  a function of \emph{all} the optimality gaps. It holds that $H_1 \ge H_2 \ge H_{2,\mathrm{lin}}$. Thus, the upper bound of OD-LinBAI is not worse (and often better) in its dependence on the hardness/complexity parameter. 
\end{enumerate}

\noindent\textbf{Comparisons to Peace \citep{katz2020empirical} (Also see Appendix~\ref{appendix_peace}).}
\begin{enumerate}[label = (\roman*), noitemsep, nolistsep]
\item To ensure there is only a \emph{single} arm in the final active set, the fixed-budget version of Peace requires $\gamma(\{a(i), a(1)\}) \ge 1$ for all suboptimal arms $i\ne1$ (where $\gamma(\cdot)$ is defined in \citet{katz2020empirical}). Note that this is not only a requirement for the theoretical bound but also a requirement for the \emph{feasibility} of the algorithm. If this inequality is not satisfied, the linear bandit instance needs to be “rescaled” before the algorithm is run, resulting in a larger bound on the error probability. In practice, the best arm is unknown and the rescaling factor can thus only be conservatively bounded as $\min_i \gamma(\{a(i), a(1)\}) \ge \min_{i,j} \gamma(\{a(i), a(j)\}) $. However, the latter quantity can be miniscule.  In particular, if there exist two arms that are nearly identical, i.e., $\min_{i,j} \gamma(\{a(i), a(j)\})$ is very small, the bound on the error probability may be larger than $1$, and hence vacuous. Besides, the algorithm may terminate with most of its time budget wasted. In contrast, OD-LinBAI is \emph{fully parameter-free} and does not require any information about the instance.

\item It is not straightforward to compare the error probabilities of OD-LinBAI and Peace in general since Peace involves some tricky terms that do not admit closed-form expressions. Here we consider the special case of standard multi-armed bandits (as discussed in Remark~\ref{rmk:standard}) with all optimality gaps equal to the minimal one $\Delta_1$. In this case $\rho^*=\Theta(\Delta_1^{-2}\cdot d)$, $\gamma^*=\Theta(\Delta_1^{-2}\cdot d\log d)$  and $\log (\gamma(\mathcal{Z}))=\Theta(\log d)$; these terms appear in the denominator of the exponent in Peace's bound on the error probability. Therefore, the error probability of Peace is $\exp \big(-\Omega( \frac{T\Delta_1^2}{d \log^2d })\big)$ while ours is $\exp \big(-\Omega( \frac{T\Delta_1^2}{d \log d })\big)$, which also shows Peace is \emph{not} minimax optimal in the exponent in view of our lower bounds, to be presented in Section~\ref{subsection_lower}. See Appendix~\ref{appendix_peace} for the precise details of the above derivations.
\end{enumerate}

\noindent\textbf{Comparisons to LinearExploration \citep{alieva2021robust} and GSE \citep{azizi2021fixed}.}
\begin{enumerate}[label = (\roman*), noitemsep, nolistsep]
\item The idea of elimination has been well-received and is ubiquitous in linear bandits. Although LinearExploration \citep{alieva2021robust}, GSE \citep{azizi2021fixed} and OD-LinBAI all leverage this idea, we emphasize that the elimination criteria for these algorithms are different. In particular, OD-LinBAI divides the time budget into roughly $\log_2d$ phases while the other algorithms divide the budget into roughly $\log_2K$ phases. Additionally, OD-LinBAI always controls the \emph{dimension} of the active set in each phase, using the dimensionality reduction techinique in Section~\ref{section_setup}.

\item The error probabilities of LinearExploration and GSE are upper bounded by $\exp \big(-\Omega( \frac{T}{\tilde H_2 \log_2 K })\big)$ and $\exp \big(-\Omega( \frac{T\Delta_1^2}{d \log_2 K })\big)$ respectively. Note that $K\ge d$, $H_{2,\mathrm{lin}} \le d / \Delta_1^2$, and the hardness quantity $\tilde H_2$ in \citet{alieva2021robust} is of the same order as $H_{2,\mathrm{lin}}$. Hence, our exponent of the bound on the error probability is an improvement over their exponents by a factor of $\Theta((\log_2 K)/(\log_2 d))$, which may be much larger than $1$.
\end{enumerate}

\subsection{Lower bound}
\label{subsection_lower}
Before stating the lower bound formally, we introduce
\begin{equation*}
   H_{1,\mathrm{lin}} = \sum _{1\le i\le d} \Delta_i^{-2}.
\end{equation*}This quantity is a generalization of $H_1$ that characterizes the difficulty of a linear bandit instance. 
This parameter is also associated with the top $d$ arms similarly to $H_{2,\mathrm{lin}}$. See Table~\ref{table_H} for a thorough comparison on different hardness quantities. 


\begin{table}
  \caption{Comparisons of different hardness quantities: $H_1$, $H_2$, $H_{1,\mathrm{lin}}$ and $H_{2,\mathrm{lin}}$.}
  \label{table_H}
  \centering
  \begin{tabular}{lll}
    \toprule
    $H_{1} = \sum _{1\le i\le K} \Delta_i^{-2}$ & $H_{2} = \max_{2\leq i \leq K} {i}\cdot {\Delta_i^{-2}}$  & $1\le H_{1}/H_{2} \le \log (2K)$ \citep{audibert2010best}     \\
    \midrule
    $H_{1,\mathrm{lin}} = \sum _{1\le i\le d} \Delta_i^{-2}$ & $H_{2,\mathrm{lin}} = \max_{2\leq i \leq d} {i}\cdot {\Delta_i^{-2}}$  & $1\le H_{1,\mathrm{lin}}/H_{2,\mathrm{lin}} \le \log (2d) $     \\
    \midrule
    $1 \le H_{1}/H_{1,\mathrm{lin}} \le K/ d $     & $1 \le H_{2}/H_{2,\mathrm{lin}} \le K /d $     &  \\
    \bottomrule
  \end{tabular}
\end{table}

For any linear bandit instance $\nu \in \mathcal E$, we denote the hardness quantity $H_{1,\mathrm{lin}}$ of $\nu$ as $H_{1,\mathrm{lin}}(\nu)$.\footnote{When there is no ambiguity, $H_{1,\mathrm{lin}}$ will also be used.} In addition, let $\mathcal E(h)$ denote the set of linear bandit instances in $\mathcal E$ whose hardness parameter $H_{1,\mathrm{lin}}$ is upper bounded by $h$ (for some $h>0$), i.e., $\mathcal E(h) = \{\nu \in \mathcal E:  H_{1,\mathrm{lin}}(\nu) \le h\}$.

\begin{theorem}
\label{theorem_lowerbound}
If $T \ge h^{2} \log (6 T d) /900$, then 
\begin{equation*}
    \min_{\Pi} \max_{\nu \in \mathcal E(h)} \Pr\left[i_{\mathrm{out}}^{\Pi} \neq 1 \right] \ge \frac 1 6 \exp \left( - \frac {240T} h \right).
\end{equation*}Further if $h \ge 15 d^2$, then 
\begin{equation*}
    \min_{\Pi} \max_{\nu \in \mathcal E(h)} \left( \Pr\left[i_{\mathrm{out}}^{\Pi} \neq 1 \right] \cdot \exp \left(  \frac {2700T} {H_{1,\mathrm{lin}}(\nu) \log_2d } \right) \right)\ge \frac 1 6.
\end{equation*}
\end{theorem}

The proof of Theorem~\ref{theorem_lowerbound} is deferred to Appendix~\ref{appendix_lowerbound}. We emphasize that even though the proof of the lower bound follows some  common ideas behind the proofs of most minimax lower bounds in bandit algorithms for various purposes, its value does not lie  in its technical novelty, but rather that the result is  {\em tight} vis-\`a-vis the upper bound we have derived based on the OD-LinBAI algorithm.  The usual strategy, which is the strategy we adopt here, is to construct and analyze specific hard instances. In particular, we leverage the instances in \citet{carpentier2016tight} for standard multi-armed bandits to construct hard linear bandit instances for any arbitrary $K$ and $d$. 
We discuss the  tightness of the lower bound in the following.


Theorem~\ref{theorem_lowerbound} first shows that for any best arm identification algorithm $\Pi$, even with the knowledge of an upper bound $h$ on the hardness quantity $H_{1, \mathrm{lin}}$, there exists a linear bandit instance such that the error probability is at least
\begin{equation}
\label{equation_lower_1}
   \exp\left(-O\left( \frac{T}{h}\right) \right).
\end{equation}Furthermore, for any best arm identification algorithm $\Pi$, without the knowledge of an upper bound $h$ on the hardness quantity $H_{1, \mathrm{lin}}$, there exists a linear bandit instance $\nu$ such that the error probability is at least
\begin{equation}
\label{equation_lower_2}
   \exp\left(-O\left( \frac{T}{H_{1,\mathrm{lin}}(\nu) \log_2d }\right) \right).
\end{equation}Comparing the lower bounds (\ref{equation_lower_1}) and (\ref{equation_lower_2}) in two different settings, we show that the agent has to pay a price of $\log_2 d$ in the absence of the knowledge about the problem complexity. Finding a best arm identification algorithm that matches the lower bound (\ref{equation_lower_1}) remains an open problem since the upper bound of BayesGap (\ref{equation_upper_BayesGap}) involves $H_1$ but not $H_{1, \mathrm{lin}}$. However, notice that the knowledge about the complexity quantity which is required for BayesGap is usually unavailable in real-life applications.

Now we compare the upper bound on the error probability of OD-LinBAI in (\ref{equation_upper_OD}) with the lower bound (\ref{equation_lower_2}). Table~\ref{table_H} shows that $H_{1, \mathrm{lin}} \ge H_{2, \mathrm{lin}}$ always holds. Therefore, the upper bound in~(\ref{equation_upper_OD}) is  not larger than the lower bound in~(\ref{equation_lower_2}) in the exponent up to absolute constants. This shows OD-LinBAI (Algorithm~\ref{algo1}) is \emph{minimax optimal} up to multiplicative factors   in the exponent and the upper bound cannot be improved in an order-wise sense in the exponent in general. At the same time, note that the upper bound holds for {\em all} instances while the lower bound is a minimax result which holds for {\em specific} instances. Since an upper bound can never be smaller than a lower bound, we know that the difficult instances for the problem of  best arm identification in linear bandits in the fixed-budget setting are those whose $H_{1, \mathrm{lin}}$ and $H_{2, \mathrm{lin}}$ are of the same order.

\section{Numerical experiments}
\label{section_exp}

In this section, we evaluate the performance of our algorithm OD-LinBAI and compare it with Sequential Halving \citep{karnin2013almost}, BayesGap \citep{hoffman2014correlation}, Peace \citep{katz2020empirical}, LinearExploration \citep{alieva2021robust} and GSE \citep{azizi2021fixed}.  
For BayesGap, there are two versions: one is BayesGap-Oracle, which is given the exact information of the required hardness quantity $H_1$; the other is BayesGap-Adaptive, which adaptively estimates the hardness quantity by the three-sigma rule. In each setting, the reported error probabilities of different algorithms are averaged over $1024$ independent trials and the (tiny) error bars indicate the standard errors of the error probabilities. 
We present the results of one synthetic dataset here.
Additional implementation details and numerical results (including another synthetic dataset, one real-world dataset and comparison to the recent LT\&S algorithm for best arm identification in linear bandits with fixed confidence \citep{jedra2020optimal}) are provided in Appendix~\ref{appendix_exp}. 

\subsection{Synthetic dataset 1: a hard instance} \label{subsec:syn1}

\begin{wrapfigure}{r}{0.6\textwidth}
	\centering
	\vspace{-0.2cm}
    \begin{minipage}[t]{0.495\linewidth}
		\centering
		\includegraphics[width=1\textwidth]{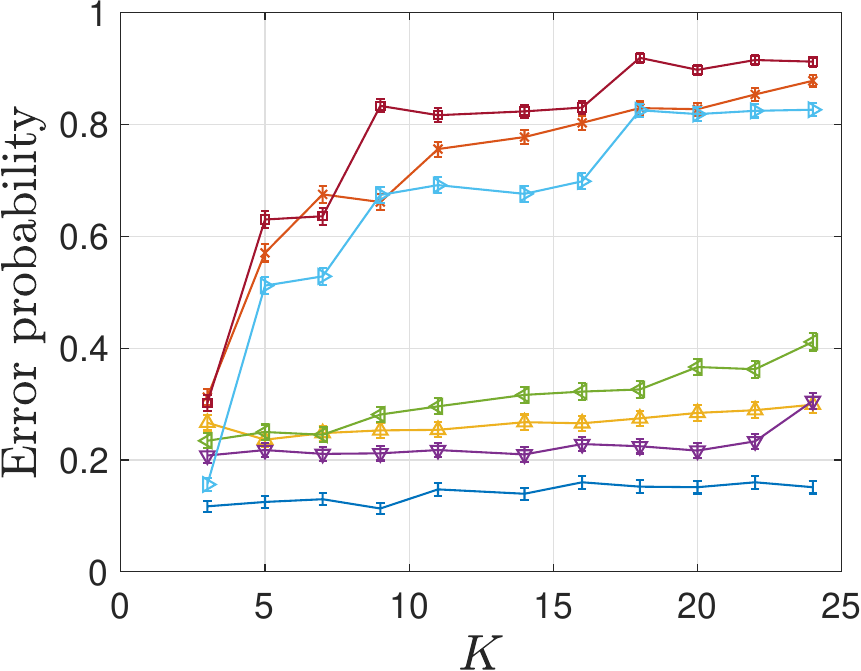}
	\end{minipage}
	\begin{minipage}[t]{0.495\linewidth}
		\centering
		\includegraphics[width=1\textwidth]{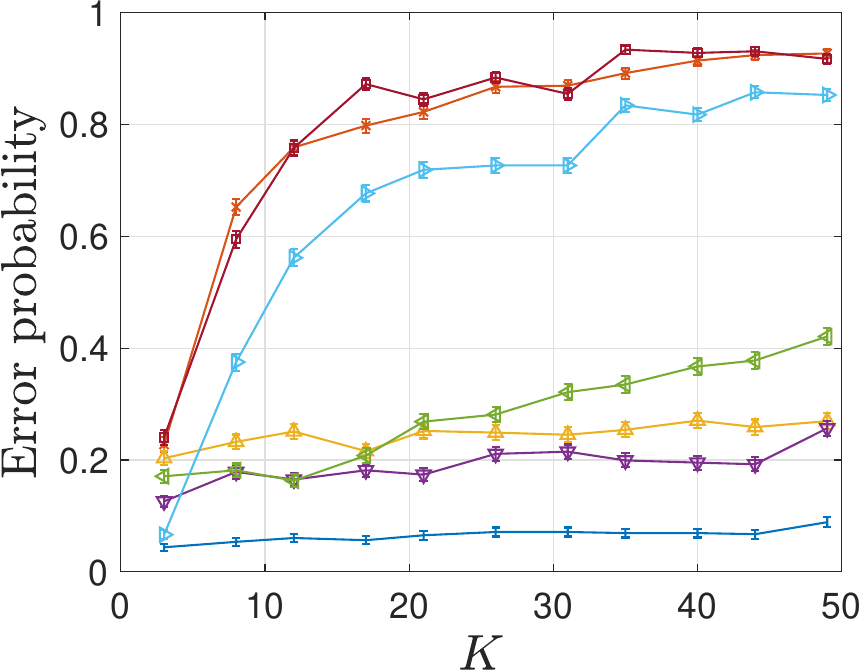}
	\end{minipage}
	\begin{minipage}[t]{1\linewidth}
		\centering
		\includegraphics[width=1\textwidth]{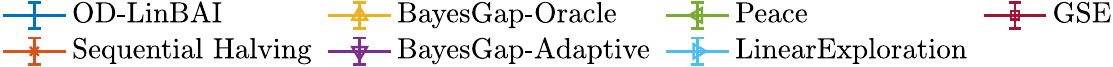}
	\end{minipage}
	\vspace{-0.15cm}
	\caption{Error probabilities for different numbers of arms $K$ with  $T=25, 50$ from left to right.}
	\label{fig_dataset1_1}
	\centering
	\vspace{0.3cm}
    \begin{minipage}[t]{0.495\linewidth}
		\centering
		\includegraphics[width=1\textwidth]{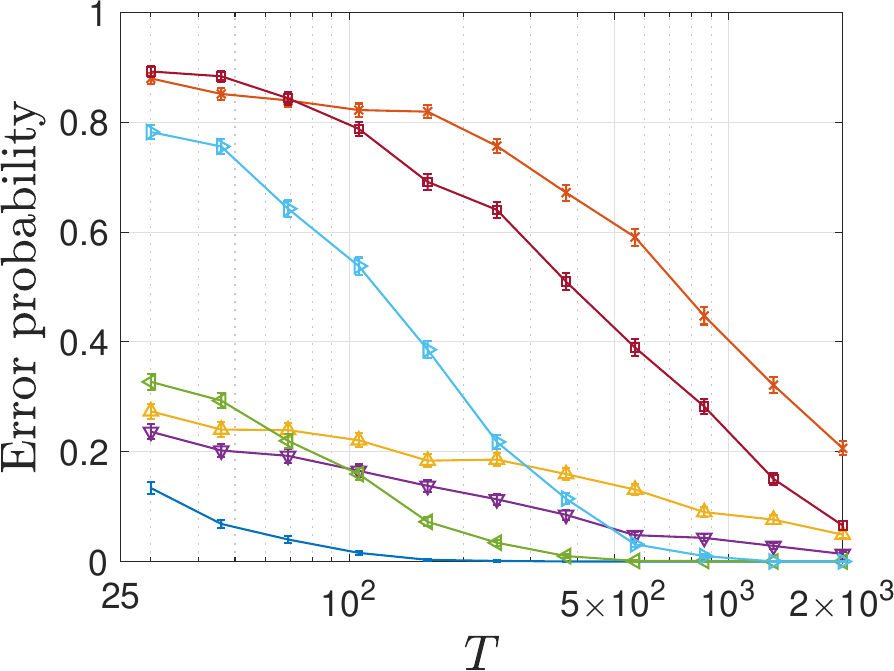}
	\end{minipage}
	\begin{minipage}[t]{0.495\linewidth}
		\centering
		\includegraphics[width=1\textwidth]{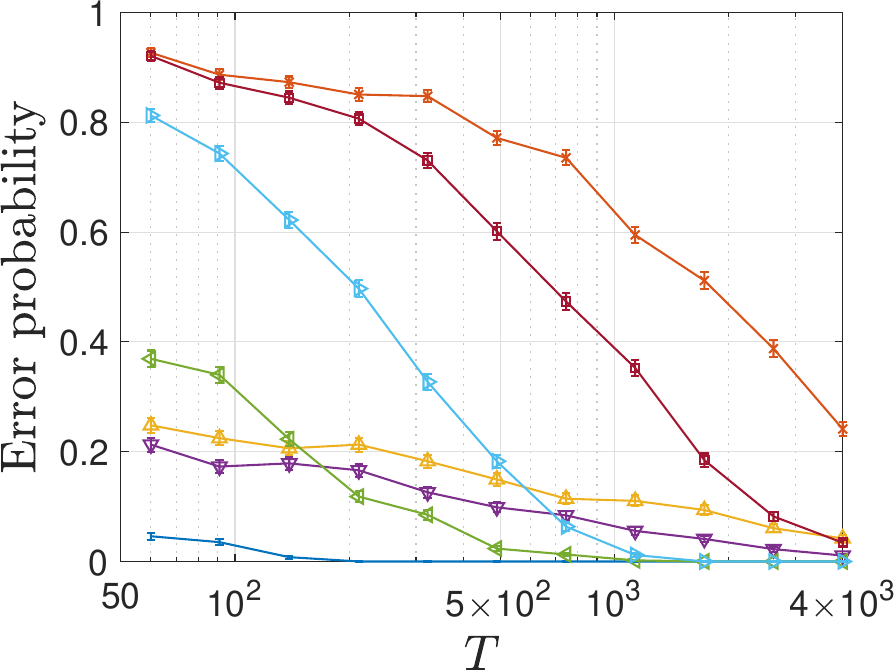}
	\end{minipage}
	\begin{minipage}[t]{1\linewidth}
		\centering
		\includegraphics[width=1\textwidth]{legend.pdf}
	\end{minipage}
	\vspace{-0.15cm}
	\caption{Error probabilities for different time budgets $T$ with  $K=25, 50$ from left to right.}
	\label{fig_dataset1_2}
	\vspace{-0.5cm}
\end{wrapfigure}

This benchmark dataset, in which there are numerous competitors for the second best arm, was considered for the problem of best arm identification in linear bandits in the fixed-confidence setting \citep{zaki2019towards, fiez2019sequential,jedra2020optimal}. Similarly, we consider the situation that $d = 2$ and $K\ge 3$. We assume that the additive random noise follows the standard Gaussian distribution $\mathcal{N}(0, 1)$. For simplicity, we set the unknown parameter vector $\theta^* = [1,0]^{\top}$. There is  one best arm and one worst arm, which correspond to the arm vectors $a(1)=[1,0]^{\top}$ and $a(K) = [\cos( {3\pi/4}),\sin( {3\pi/4})]^{\top}$ respectively. For any arm $i \in \{2,3,\dots,K-1\}$, the corresponding arm vector is chosen to be $a(i) = [\cos( {\pi/4+\phi_{i}}),\sin( {\pi/4}+\phi_{i})]^{\top}$ with $\phi_{i}$ drawn independently from $ \mathcal{N}(0, 0.09^2)$. Therefore, there are $K-2$ almost second best arms. Considering the definitions of four hardness quantities, it holds that  $H_{1} \approx H_{2} \approx \frac K d H_{1,\mathrm{lin}} \approx \frac K d H_{2,\mathrm{lin}}$.
Hence this is a hard instance in the sense that the linear structure is extremely strong. A good algorithm needs to fully utilize the correlations of the arms to obtain information as efficiently as possible. 

The experimental results with fixed $T$ and $K$ are presented in Figure~\ref{fig_dataset1_1} and Figure~\ref{fig_dataset1_2} respectively. In terms of this hard linear bandit instance, OD-LinBAI is clearly superior compared to its competitors. In fact, OD-LinBAI consistently pulls only one arm from the $K-2$ almost second best arms and thus suffers minimal impact from the increase in $K$.


\section{Conclusions and Future Work}
\label{section_conclusion}

We introduce the G-optimal design technique into the problem of best arm identification in linear bandits in the fixed-budget setting. We  design a parameter-free and efficient algorithm OD-LinBAI. To characterize the difficulty of a linear bandit instance, we introduce two hardness quantities $H_{1,\mathrm{lin}}$ and $H_{2,\mathrm{lin}}$. The upper bound of the error probability of OD-LinBAI and the minimax lower bound of this problem are respectively characterized by $H_{1,\mathrm{lin}}$ and $H_{2,\mathrm{lin}}$ instead of their analogues $H_1$ and $H_2$ in standard multi-armed bandits. For the first time,  minimax optimality  (up to constant multiplicative factors in the exponent) has been achieved in this problem. While we submit that the ingredients that constitute OD-LinBAI are not surprising in the bandit literature, an open problem thus far has hence been  solved in this contribution (by the careful derivation of an upper bound on the error probability of OD-LinBAI and an accompanying minimax lower bound). Our theoretical findings are  also supported by the considerable improvements of the empirical performance of OD-LinBAI vis-\`a-vis  existing algorithms on benchmark datasets.

A direction for future work is to design an instance-dependent asymptotically optimal algorithm for this problem. However, finding  such an algorithm or  an instance-dependent asymptotic lower bound for the problem of best arm identification in standard (i.e., $K$-armed) multi-armed bandits in the fixed-budget setting remains open. 
Finally, as Thompson sampling \citep{thompson1933likelihood, agrawal2012analysis} has been successfully extended to pure exploration in standard multi-armed bandits \citep{russo2016simple, shang2020fixed, qin2022adaptivity,jourdan2022top}, it is interesting to study whether this technique can be generalized to {\em  linear} bandits,  in both  the fixed-budget and fixed-confidence settings.

\begin{ack}
This research/project is supported by the National Research Foundation Singapore and DSO National Laboratories under the AI Singapore Programme (AISG Award No: AISG2-RP-2020-018)
  and by Singapore Ministry of Education (MOE) AcRF Tier 1 Grants (A-0009042-01-00 and A-8000189-01-00).
\end{ack}

\newpage
\bibliographystyle{unsrtnat}
\bibliography{references}

\newpage
\appendix

\section{More discussions on the G-optimal design}
\label{appendix_goptimal}
\paragraph{$\epsilon$-approximate G-optimal design.} For the problem of best arm identification in linear bandits in the fixed-budget setting, it is sufficient to compute an $\epsilon$-approximate G-optimal design with minimal impact on performance. For an $\epsilon$-approximate optimal design $\pi$, $g(\pi) \le (1+\epsilon)d$. \citet{todd2016minimum} shows that such a design can be computed within $4d(\log\log d +7/2)+28d/\epsilon$ iterations by the Frank--Wolfe algorithm with a specific initialization. If we only compute $\epsilon$-approximate G-optimal designs in OD-LinBAI (Algorithm~\ref{algo1}), the upper bound on the error probability will only deteriorate by a factor of $(1+\epsilon)$ as follows.

\begin{theorem}
\label{theorem_appendix1}
For any linear bandit instance $\nu \in \mathcal E$, OD-LinBAI, using $\epsilon$-approximate G-optimal designs, outputs an arm $i_{\mathrm{out}}$ satisfying
$$
\Pr\left[i_{\mathrm{out}} \neq 1 \right] \le
\left( \frac {4K} { d}+3\log _{2} d\right) \exp \left( - \frac {m} {32(1+\epsilon) H_{2,\mathrm{lin}}} \right)
$$
where  $m$ is defined in Equation~(\ref{equation_m}).
\end{theorem}

\paragraph{Near-optimal design with smaller support.} Recently, a near-optimal design with smaller support was proposed in \citet{lattimore2020learning}. In detail, there exists a design $\pi: \{a(i):i\in \mathcal A \} \rightarrow[0,1]$ such that $g(\pi) \le 2d$ and
$|\operatorname{Supp}(\pi)| \leq 4d(\log\log d +11)$. \citet{todd2016minimum} shows that such a design can be computed within $4d(\log\log d +21/2)$ iterations by the Frank--Wolfe algorithm with a specific initialization. Since the support of the design is smaller when $d$ is large, we can choose a larger $m$ in OD-LinBAI while the total budget consumed by the agent is still bounded by $T$. In particular, we can choose the parameter $m$ as 
\begin{equation}
\label{equation_m2}
    m = \frac {T-\min(K, 4d(\log\log d +11)) -\sum\limits_{r=1}^{{\lceil\log _{2}  d \rceil}-1} {\left\lceil\frac {d} {2^{r}}\right\rceil}} {{\lceil\log _{2}  d \rceil} }.
\end{equation}
The error probability can be bounded as follows.
\begin{theorem}
\label{theorem_appendix2}
For any linear bandit instance $\nu \in \mathcal E$, OD-LinBAI, using near-optimal designs with smaller support, outputs an arm $i_{\mathrm{out}}$ satisfying
$$
\Pr\left[i_{\mathrm{out}} \neq 1 \right] \le
\left( \frac {4K} { d}+3\log _{2} d\right) \exp \left( - \frac {m} {64 H_{2,\mathrm{lin}}} \right)
$$
where  $m$ is defined in Equation~(\ref{equation_m2}).
\end{theorem}

The proofs of Theorem~\ref{theorem_appendix1} and Theorem~\ref{theorem_appendix2} are very similar to Theorem~\ref{theorem_upperbound} and thus omitted; they only involve plugging the result of the approximate optimal design into the upper bound.

\section{Proof of Theorem~\ref{theorem_upperbound}}
\label{appendix_upperbound}
Before going to the proof of Theorem~\ref{theorem_upperbound}, we first introduce some useful lemmas. Lemma~\ref{lemma_m} shows Algorithm~\ref{algo1} is feasible in the sense that the total budget consumed by the agent is no more than $T$, and $i_{\mathrm{out}}$ is well-defined.

\begin{lemma} 
\label{lemma_m} 
With parameter $m$ defined as Equation~(\ref{equation_m}),  Algorithm~\ref{algo1} terminates in phase $\lceil\log _{2}  d\rceil$ with no more than a total of $T$ arm pulls. 
\end{lemma}
\begin{proof}
When $d = 2$, Algorithm~\ref{algo1} terminates in one phase. When $ d > 2$, by the property of ceiling function, we have $\frac 1 2 <\frac { d} {2^{\lceil\log _{2}  d \rceil}}\le 1$.
Thus, the number of arms in $ \mathcal A_{\lceil\log _{2} d \rceil-1}$ is ${\left\lceil\frac { d} {2^{\lceil\log _{2}  d \rceil-1}}\right\rceil}=2$, while the number of arms in $ \mathcal A_{\lceil\log _{2} d \rceil}$ is ${\left\lceil\frac { d} {2^{\lceil\log _{2}  d \rceil}}\right\rceil}=1$. As a result, Algorithm~\ref{algo1} always terminates in phase $\lceil\log _{2}  d\rceil$.

Now we bound the number of arm pulls. For any phase $r$, $|\operatorname{Supp}\left(\pi_{r}\right)|$ is always bounded by the cardinality of the active set $\mathcal A _{r-1}$. In particular, for the first phase, according to Theorem~\ref{theorem_goptimal}, there exists a G-optimal design $\pi_r$ with $|\operatorname{Supp}\left(\pi_{r}\right)| \leq d(d+1) / 2$.  Altogether, we have
\begin{equation*}
    |\operatorname{Supp}\left(\pi_{r}\right)| \le \begin{cases}
    \min(K, \frac { d ( d +1)} 2) & \text{when } r= 1 \\
    {\left\lceil\frac { d} {2^{r-1}}\right\rceil} & \text{when } r>1.
    \end{cases}
\end{equation*}
Then the number of total arm pulls is bounded as 
\begin{align}
     \sum_{r=1}^{\lceil\log _{2} d\rceil} T_r &= \sum_{r=1}^{\lceil\log _{2} d\rceil} \sum_{i \in  \mathcal A_r} T_r(i) \notag \\ 
     &= \sum_{r=1}^{\lceil\log _{2} d\rceil} \sum_{i \in  \mathcal A_r} \left\lceil \pi_r(a_r(i)) \cdot m \right\rceil \notag \\
     & \le \sum_{r=1}^{\lceil\log _{2} d\rceil} \left (  |\operatorname{Supp}\left(\pi_{r}\right)| + \sum_{i \in  \mathcal A_r}  \pi_r(a_r(i)) \cdot m   \right)  \label{lemma_m_1}\\
     &\le \min\left(K, \frac {d (d +1)} 2 \right) +\sum\limits_{r=2}^{{\lceil\log _{2} d \rceil}} {\left\lceil\frac {d} {2^{r-1}}\right\rceil} + \lceil\log _{2} d\rceil \cdot m \notag \\
     &= T \label{lemma_m_2}
\end{align}
where line (\ref{lemma_m_1}) follows from the property of ceiling function and line (\ref{lemma_m_2}) follows from the definition of $m$.
\end{proof}

Lemma~\ref{lemma1} bounds the probability that a certain arm has its estimate of the expected reward larger than that of the best arm in single phase $r$.
\begin{lemma} 
\label{lemma1} For a fixed realization of $\mathcal{A}_{r-1}$ satisfying $1\in\mathcal{A}_{r-1}$, for any arm $i \in \mathcal{A}_{r-1}$,
\begin{equation*}
\Pr\left[\hat{p}_{r}(1)<\hat{p}_{r}(i)\right] \leq \exp \left(-\frac{m \Delta_{i}^{2}}{8{\left\lceil\frac {d} {2^{r-1}}\right\rceil}} \right).
\end{equation*}
\end{lemma}

\begin{proof}
Let $\theta_r^*$ denote the corresponding unknown parameter vector for the dimensionality-reduced arm vectors $\{a_r(i):i\in \mathcal A_{r-1}\}$. Also we set 
\begin{equation*}
    V_r(\pi_r) = \sum_{i \in \mathcal{A}_{r-1}}  \pi_r(a_r(i)) a_r(i) a_r(i)^{\top}.
\end{equation*} 
Then we have
\begin{align}
    &\phantom{\;=\;}\Pr\left[\hat{p}_{r}(1)<\hat{p}_{r}(i) \right] \notag \\
    &=\Pr\left[ \braket{\hat \theta_r -\theta^*_r,a_r(1)-a_r(i)} < -\Delta_i  \right] \label{lemma1_1} \\
    &\leq\exp \left(-\frac {\Delta_i^2} {2\|a_r(1)-a_r(i)\|_{V_r^{-1}}^2} \right) \label{lemma1_2} \\
    &\leq \exp \left(-\frac {\Delta_i^2} {8 \max_{i\in \mathcal A_r}\|a_r(i)\|_{V_r^{-1}}^2} \right) \label{lemma1_3}\\
    &\leq \exp \left(-\frac {\Delta_i^2\cdot m} {8 \max_{i\in \mathcal A_r}\|a_r(i)\|_{V_r(\pi_r)^{-1}}^2} \right) \label{lemma1_33}\\
    &= \exp \left(-\frac {m\Delta_i^2} {8 d_r} \right) \label{lemma1_4}\\
    &\leq \exp \left(-\frac{m\Delta_{i}^{2}}{8{\left\lceil\frac {d} {2^{r-1}}\right\rceil}}  \right). \label{lemma1_5}
\end{align} 
Line (\ref{lemma1_1}) follows from
 \begin{equation*}
 \begin{cases}
 \hat{p}_{r}(1) = \braket{\hat \theta_r ,a_r(1)} \\
 \hat{p}_{r}(i) = \braket{\hat \theta_r ,a_r(i)} \\
\Delta_i =  \braket{\theta^*_r,a_r(1)-a_r(i)} = \braket{\theta^*,a(1)-a(i)}.
\end{cases}
\end{equation*}
Line (\ref{lemma1_2}) follows from Proposition~\ref{prop_concentration_estimator}, the confidence bound for the OLS estimator.

Line (\ref{lemma1_3}) follows from the triangle inequality for $\|\cdot\|_{V_r^{-1}}$ norm.

Line (\ref{lemma1_33}) follows from
\begin{align*}
    \|a_r(i)\|_{V_r^{-1}}^2 &= a_r(i)^{\top} V_r^{-1}a_r(i) \\
    &= a_r(i)^{\top} \left( \sum_{j \in \mathcal{A}_{r-1}} T_{r}(j) a_r(j) a_r(j)^{\top} \right)^{-1} a_r(i) \\
    &\le a_r(i)^{\top} \left( \sum_{j \in \mathcal{A}_{r-1}} m \pi_r(a_r(j)) a_r(j) a_r(j)^{\top} \right)^{-1} a_r(i) \\
    &= \frac 1 m a_r(i)^{\top} \left( \sum_{j \in \mathcal{A}_{r-1}}  \pi_r(a_r(j)) a_r(j) a_r(j)^{\top} \right)^{-1} a_r(i) \\
    &= \frac 1 m a_r(i)^{\top} V_r(\pi_r)^{-1} a_r(i) \\ 
    &= \frac 1 m {\|a_r(i)\|_{V_r(\pi_r)^{-1}}^2} . 
\end{align*}

Line (\ref{lemma1_4}) follows from Theorem~\ref{theorem_goptimal}, the property of G-optimal design.

Line (\ref{lemma1_5}) follows from the fact that the dimension of the space spanned by the corresponding arm vectors of the active arm set $\mathcal A _{r-1}$ is  not larger than the cardinality of $\mathcal A _{r-1}$. 
\end{proof}

Armed with Lemma~\ref{lemma1}, then we bound the error probability of single phase $r$ in Lemma~\ref{lemma2}.
\begin{lemma}
\label{lemma2}Assume that the best arm is not eliminated prior to phase $r$, i.e., $1\in\mathcal{A}_{r-1}$. Then the probability that the best arm is eliminated in phase $r$ is bounded as 
$$ 
 \Pr[1 \notin \mathcal A _{r}\mid  1 \in \mathcal A_{r-1} ] \leq
\begin{cases}
\frac {4K} {d} \exp \left(-\frac{m \Delta_{i_r}^{2}}{32i_r} \right) &\text{when }r=1\\
3 \exp \left(-\frac{m \Delta_{i_r}^{2}} {32i_r} \right)&\text{when }r>1
\end{cases}
$$
where $i_r = {\left\lceil\frac {d} {2^{r+1}}\right\rceil+1}$.
\end{lemma}

\begin{proof}
First, as Lemma~\ref{lemma1}, we conditioned on the specific realization of $\mathcal{A}_{r-1}$ such that $1\in\mathcal{A}_{r-1}$.

Define $\mathcal B_{r}$ as the set of arms in $\mathcal A_{r-1}$ excluding the best arm and ${\left\lceil\frac {d} {2^{r+1}}\right\rceil}- 1$  suboptimal arms with the largest expected rewards. Therefore, we have $|\mathcal B_r| = |\mathcal A_{r-1}|-{\left\lceil\frac {d} {2^{r+1}}\right\rceil}$ and $\min_{i\in\mathcal B_r} \Delta_{i} \ge \Delta_{\left\lceil\frac {d} {2^{r+1}}\right\rceil+1}$.

If the best arm is eliminated in phase $r,$ then at least ${\left\lceil\frac {d} {2^{r}}\right\rceil}-{\left\lceil\frac {d} {2^{r+1}}\right\rceil}+1$ arms of $\mathcal B_{r}$ have their estimates of the expected rewards larger than that of the best arm.

Let $N_{r}$ denote the number of arms in $\mathcal B_r$ whose estimates of the expected rewards larger than that of the best arm. By Lemma~\ref{lemma1}, we have
\begin{align*}
    \E \left[ N_r  \right] &= \sum_{i\in\mathcal B_r} \Pr\left[\hat{p}_{r}(1)<\hat{p}_{r}(i)\right]  \\
    &\leq  \sum_{i\in\mathcal B_r} \exp \left(-\frac{m \Delta_{i}^{2}}{8{\left\lceil\frac {d} {2^{r-1}}\right\rceil}} \right) \\
    &\leq |\mathcal B_r| \max_{i\in\mathcal B_r}\exp \left(-\frac{m \Delta_{i}^{2}}{8{\left\lceil\frac {d} {2^{r-1}}\right\rceil}} \right) \\
    &\leq \left(|\mathcal A_{r-1}|-{\left\lceil\frac {d} {2^{r+1}}\right\rceil}\right) \exp \left(-\frac{m \Delta_{\left\lceil\frac {d} {2^{r+1}}\right\rceil+1}^{2}}{8{\left\lceil\frac {d} {2^{r-1}}\right\rceil}} \right) \\
    &\leq \left(|\mathcal A_{r-1}|-{\left\lceil\frac {d} {2^{r+1}}\right\rceil}\right) \exp \left(-\frac{m \Delta_{\left\lceil\frac {d} {2^{r+1}}\right\rceil+1}^{2}}{32\left({\left\lceil\frac {d} {2^{r+1}}\right\rceil}+1\right)} \right).
\end{align*}
Then, together with Markov’s inequality, we obtain
\begin{align*}
    \Pr[1 \notin \mathcal A _r  ] &\leq \Pr\left[N_r \ge  {\left\lceil\frac {d} {2^{r}}\right\rceil}-{\left\lceil\frac {d} {2^{r+1}}\right\rceil}+1 \right] \\
    & \leq \frac {\E \left[ N_r  \right]} {{\left\lceil\frac {d} {2^{r}}\right\rceil}-{\left\lceil\frac {d} {2^{r+1}}\right\rceil}+1} \\
    &\leq \frac {|\mathcal A_{r-1}|-{\left\lceil\frac {d} {2^{r+1}}\right\rceil}}{{\left\lceil\frac {d} {2^{r}}\right\rceil}-{\left\lceil\frac {d} {2^{r+1}}\right\rceil}+1} \exp \left(-\frac{m \Delta_{\left\lceil\frac {d} {2^{r+1}}\right\rceil+1}^{2}}{32\left({\left\lceil\frac {d} {2^{r+1}}\right\rceil}+1\right)} \right).
\end{align*}
When $r=1$, we have $|\mathcal A_{r-1}| = K$. Thus,
\begin{align*}
    \frac {|\mathcal A_{r-1}|-{\left\lceil\frac {d} {2^{r+1}}\right\rceil}}{{\left\lceil\frac {d} {2^{r}}\right\rceil}-{\left\lceil\frac {d} {2^{r+1}}\right\rceil}+1}   &=  \frac {K-{\left\lceil\frac {d} {2^{r+1}}\right\rceil}}{{\left\lceil\frac {d} {2^{r}}\right\rceil}-{\left\lceil\frac {d} {2^{r+1}}\right\rceil}+1}  \\
    &\leq  \frac K {{\frac {d} {2}}-{\frac {d} {2^{2}}}} \\
    &= \frac {4K} {d}.
\end{align*} 

When $r>1$, we have $|\mathcal A_{r-1}| = {\left\lceil\frac {d} {2^{r-1}}\right\rceil}$. Thus, 
\begin{align*}
    \frac {|\mathcal A_{r-1}|-{\left\lceil\frac {d} {2^{r+1}}\right\rceil}}{{\left\lceil\frac {d} {2^{r}}\right\rceil}-{\left\lceil\frac {d} {2^{r+1}}\right\rceil}+1}   &=  \frac {{\left\lceil\frac {d} {2^{r-1}}\right\rceil}-{\left\lceil\frac {d} {2^{r+1}}\right\rceil}}{{\left\lceil\frac {d} {2^{r}}\right\rceil}-{\left\lceil\frac {d} {2^{r+1}}\right\rceil}+1} 
    \\&\leq\frac {{\frac {d} {2^{r-1}}}+1-{\left\lceil\frac {d} {2^{r+1}}\right\rceil}}{{\frac {d} {2^{r}}}-{\left\lceil\frac {d} {2^{r+1}}\right\rceil}+1}  \\
    &\leq \frac {3\cdot{\frac {d} {2^{r+1}}}+{\frac {d} {2^{r+1}}}+1-{\left\lceil\frac {d} {2^{r+1}}\right\rceil}}{{\frac {d} {2^{r+1}}}+{\frac {d} {2^{r+1}}}+1-{\left\lceil\frac {d} {2^{r+1}}\right\rceil}} \\
    &\leq  3
\end{align*}
where the last inequality results from the fact that for any $x,y>0$, $\frac {3x+y}{x+y} \le 3$ .

Therefore, for this specific realization of $\mathcal{A}_{r-1}$ satisfying $1\in\mathcal{A}_{r-1}$, 
$$
    \Pr[1 \notin \mathcal A _{r}] \leq
\begin{cases}
\frac {4K} {d} \exp \left(-\frac{m \Delta_{i_r}^{2}}{32i_r} \right) &\text{when }r=1\\
3 \exp \left(-\frac{m \Delta_{i_r}^{2}} {32i_r} \right)&\text{when }r>1
\end{cases}
$$
where $i_r = {\left\lceil\frac {d} {2^{r+1}}\right\rceil+1}$.

Eventually, by the law of total probability, the error probability of phase $r$ conditioned on $1\in\mathcal{A}_{r-1}$ can be bounded as
$$
    \Pr[1 \notin \mathcal A _{r}\mid  1 \in \mathcal A_{r-1} ] \leq
\begin{cases}
\frac {4K} {d} \exp \left(-\frac{m \Delta_{i_r}^{2}}{32i_r} \right) &\text{when }r=1\\
3 \exp \left(-\frac{m \Delta_{i_r}^{2}} {32i_r} \right)&\text{when }r>1.
\end{cases}
$$
\end{proof}

Now we return to the proof of Theorem~\ref{theorem_upperbound}.
\begin{proof}[Proof of Theorem~\ref{theorem_upperbound}]
By applying Lemma~\ref{lemma_m} and Lemma~\ref{lemma2}, we have
\begin{align*}
\Pr\left[i_{\mathrm{out}} \neq 1 \right] &= \Pr\left[1\notin \mathcal A_{\lceil\log _{2} d \rceil} \right] \\
&\le \sum_{r=1}^{\lceil\log _{2} d \rceil}  \Pr[1 \notin \mathcal A _{r}\mid 1 \in \mathcal A_{r-1} ] \\
&\le \frac {4K} {d} \exp \left(-\frac{m \Delta_{i_1}^{2}}{32i_1} \right) + \sum_{r=2}^{\lceil\log _{2} d \rceil} 3 \exp \left(-\frac{m \Delta_{i_r}^{2}} {32i_r} \right) \\
&\le \left( \frac {4K} { d}+3\left( {\lceil\log _{2} d \rceil}-1 \right)\right) \exp \left( - \frac {m} {32} \cdot \frac 1 {\max_{2\leq i \leq d} \frac {i} {\Delta_i^2} } \right) \\
&<
\left( \frac {4K} { d}+3\log _{2} d\right) \exp \left( - \frac {m} {32H_{2,\mathrm{lin}}} \right)
\end{align*}
where  $H_{2,\mathrm{lin}}$ is defined as
$$
H_{2,\mathrm{lin}} = \max_{2\leq i \leq d} \frac {i} {\Delta_i^2}.
$$
\end{proof}

\section{On the detailed comparisons to (fixed-budget) Peace \texorpdfstring{\citep{katz2020empirical}}{[16]}}
\label{appendix_peace}

In this Appendix, we show the detailed derivation of our comparisons to the fixed-budget version of Peace \citep{katz2020empirical}. 

In the fixed-budget setting, Theorem 6 in \citep{katz2020empirical} shows the error probability of Peace is upper bounded by  $$2 \lceil \log (\gamma(\mathcal{Z}))\rceil\exp\left(-\frac T {c(\rho^*+\gamma^*)\log (\gamma(\mathcal Z))}\right)$$
with a constant $c$. $\rho^*$, $\gamma^*$ and $\gamma(\mathcal Z)$ are defined therein and replicated below with the notations of this paper for the sake of clarity and completeness. For comparison to our bound in Theorem~\ref{theorem_upperbound}, we only focus on the exponential term with respect to time budget $T$ (i.e., we ignore the pre-exponential term $2 \lceil \log (\gamma(\mathcal{Z}))\rceil$).  We assume that $T$ is large so the exponential term dominates the exponential decay rate  of the bound on the error probability. 


Then we consider the special case of standard multi-armed bandits with $d=K$ (as discussed in Remark~\ref{rmk:standard}) and  all optimality gaps equal to the smallest one $\Delta_1$. For the term $\rho^* $, we have
$$
\rho^* = \min_{\pi } \rho^*(\pi)
$$
where 
$$
\begin{aligned}
\rho^*(\pi) &= \max_{i \in \mathcal{A}\setminus \{1\} } \frac{ \| a(1)-a(i)\|^2_{V(\pi)^{-1}} }{\Delta_1^{2}}\\&=\frac 1 {\Delta_1^{2}}\max_{i \in \mathcal{A}\setminus \{1\} }\left ( \frac 1{ \pi(a(1))}+\frac 1{ \pi(a(i))}\right)
\\&=\frac 1 {\Delta_1^{2}}\left ( \frac 1{ \pi(a(1))}+\max_{i \in \mathcal{A}\setminus \{1\} }\frac 1{ \pi(a(i))}\right).
\end{aligned}
$$

Thus it is straightforward to see that
$$
\begin{aligned}
\rho^* &= \min_{\pi} \frac 1 {\Delta_1^{2}}\left ( \frac 1{ \pi(a(1))}+\max_{i \in \mathcal{A}\setminus \{1\} }\frac 1{ \pi(a(i))}\right) \\
&= \Delta_1^{-2}(d + 2\sqrt{d-1}) \\
&= \Theta(\Delta_1^{-2}\cdot d),
\end{aligned}
$$
where the minimum is attained at the distribution $\pi^* =\frac{1}{1+\sqrt{d-1}}\left(1, \frac{1}{\sqrt{d-1}}, \ldots,  \frac{1}{\sqrt{d-1}}\right)$.

For the term $\gamma^*$, we have 
\begin{equation}
\gamma^* = \min_{\pi} \gamma^*(\pi) \label{eqn:gamma_star}
\end{equation}
where
$$
\begin{aligned}
\gamma^*(\pi) &= \left(\mathbb{E}_{\eta\sim \mathcal N (0,I)} \left[ \max_{i \in \mathcal{A}\setminus \{1\} } \frac{ ( a(1)-a(i))^\top {V(\pi)^{-1/2}}\eta }{\Delta_1} \right]\right)^2 \\
&= \frac 1 {\Delta_1^{2}} \left(\mathbb{E}_{\eta\sim \mathcal N (0,I)} \left[ \max_{i \in \mathcal{A}\setminus \{1\} } \frac {\eta_1}{\sqrt{\pi(a(1))}} - \frac {\eta_i}{\sqrt{\pi(a(i))}} \right] \right)^2 \\
&= \frac 1 {\Delta_1^{2}}  \left(\mathbb{E}_{\eta\sim \mathcal N (0,I)} \left[ \max_{i \in \mathcal{A}\setminus \{1\} }\frac {\eta_i}{\sqrt{\pi(a(i))}} \right]\right)^2.
\end{aligned}
$$

Recall that $\eta_1,\eta_2, \ldots, \eta_d$ are independent and identically distributed.  By symmetry, the solution $\pi^*$ to the optimization problem~\eqref{eqn:gamma_star}  must be symmetric/equal  for ${i \in \mathcal{A}\setminus \{1\} }$, i.e.,   $\pi^*(a(2))=\cdots =  \pi^*(a(d))$. In addition, it is clear that $\pi^*(1) = 0$ (to minimize the expectation). Hence, $\pi^*(a(2))=\cdots =  \pi^*(a(d))=1/(d-1)$. Using the fact that the expectation of the maximum of $d-1$ independent standard Gaussian random variables is $\Theta(\sqrt{\log (d-1)})$, we conclude that $\gamma^*=\Theta(\Delta_1^{-2}\cdot d\log d)$.

Finally, for the term $\gamma(\mathcal{Z})$, since 
$$
 \max_{i,j \in \mathcal{A} } { ( a(i)-a(j))^\top {V(\pi)^{-1/2}}\eta } \ge\max_{i \in \mathcal{A}\setminus \{1\}} { ( a(1)-a(i))^\top {V(\pi)^{-1/2}}\eta }
$$
and both 
$$
\mathbb{E}_{\eta\sim \mathcal N (0,I)} \left[ \max_{i,j \in \mathcal{A} } { ( a(i)-a(j))^\top {V(\pi)^{-1/2}}\eta } \right]
$$
and
$$
\mathbb{E}_{\eta\sim \mathcal N (0,I)} \left[ \max_{i \in \mathcal{A}\setminus \{1\}} { ( a(1)-a(i))^\top {V(\pi)^{-1/2}}\eta } \right]
$$
are nonnegative, we have 
$$
\begin{aligned}
\gamma(\mathcal{Z})&= \min_{\pi} \left(\mathbb{E}_{\eta\sim \mathcal N (0,I)} \left[ \max_{i,j \in \mathcal{A} } { ( a(i)-a(j))^\top {V(\pi)^{-1/2}}\eta } \right]\right)^2  \\
&\ge \min_{\pi} \left(\mathbb{E}_{\eta\sim \mathcal N (0,I)} \left[ \max_{i \in \mathcal{A}\setminus \{1\}} { ( a(1)-a(i))^\top {V(\pi)^{-1/2}}\eta } \right]\right)^2  \\
&= {\Delta_1^{2}} \gamma^*
\end{aligned}
$$
which shows $\log (\gamma(\mathcal{Z}))=\Omega(\log d)$.
As noted in \citet{katz2020empirical} (see lines after Theorem 6 therein), $\log (\gamma(\mathcal{Z}))=O(\log d)$ in linear bandits. Therefore, it holds that $\log (\gamma(\mathcal{Z}))=\Theta(\log d)$.

Altogether, in this special case of standard ($K$-armed) multi-armed bandits, the upper bound on the error probability of Peace \citep{katz2020empirical} writes $$\text{Peace}\le \exp \left(-\Omega \left( \frac{T\Delta_1^2}{d \; \red{\log^2d} } \right)\right)$$ (focusing  only on the exponential term) while our upper bound on the error probability (in   Theorem~\ref{theorem_upperbound}) reduces to $$\text{OD-LinBAI}\le\exp \left(-\Omega \left( \frac{T\Delta_1^2}{d\;  \blue{\log d} }\right)\right),$$ (again focusing  only on the exponential term) which clearly shows Peace is not minimax optimal in the fixed-budget setting. Peace is off by a multiplicative factor of $\log d$ in the denominator in the exponent. We note that   the term $\gamma^*$  in \eqref{eqn:gamma_star} (and not $\rho^*$) causes the overall bound of Peace to be suboptimal in the minimax sense.

\section{Proof of Theorem~\ref{theorem_lowerbound}}
\label{appendix_lowerbound}
The proof of Theorem~\ref{theorem_lowerbound} is built on the connection between linear bandits and standard multi-armed bandits \citep{carpentier2016tight}. Therefore, we first introduce the setting of best arm identification in standard multi-armed bandits.

In a standard multi-armed bandit instance $\tilde \nu$, the agent is given an arm set $\mathcal A = [K]$.  Each arm $i \in \mathcal A$ is associated with a reward distribution $P_i$ supported in $[0, 1]$, which is unknown to the agent. At each time $t$, the agent chooses an arm $A_t$ from the arm set $\mathcal A$ and then observes a stochastic reward $X_t \in [0, 1]$ drawn from $P_{A_t}$. 

In the fixed-budget setting, given a time budget $T \in \mathbb N$, the agent also aims at maximizing the probability of identifying the best arm with no more than $T$ arm pulls. More formally, the agent uses an \emph{online} algorithm $\tilde \Pi$ to decide the arm $A_t^{\tilde \Pi}$ to pull at each time step $t$, and the arm $i_{\mathrm{out}}^{\tilde \Pi} \in \mathcal A$  to output as the identified best arm by time $T$. 

\newcommand{\dif}{\mathop{}\!\mathrm{d}}
As in linear bandits, we assume that the expected rewards of the arms are in descending order and the best arm is unique. Let $\tilde {\mathcal E}$ denote the set of all the standard multi-armed bandit instances defined above. For any arm $i\in \mathcal A$, let $p(i)$ denote the expected reward under $P_i$. Similarly, for any suboptimal arm $i$, we denote $\Delta_i = p(1) - p(i)$ as the optimality gap. For ease of notation, we also set $\Delta_1= \Delta_2$.

Moreover, the two hardness quantities  $H_1$ and $H_2$ are also applicable to standard multi-armed bandits. For any standard multi-armed bandit instance $\tilde \nu \in \tilde {\mathcal E}$, we denote the hardness quantity $H_{1}$ of $\tilde \nu$ as $H_{1}(\tilde \nu)$. In addition, let $\tilde {\mathcal E}(h)$ denote the set of standard multi-armed bandit instances in $\tilde  {\mathcal E}$ whose $H_{1}$ is bounded by $h$ ($h>0$), i.e., $ \tilde {\mathcal E}(h) = \{\tilde {\nu} \in \tilde {\mathcal E}:  H_{1}(\tilde {\nu}) \le h\}$.

A minimax lower bound for the problem of best arm identification in standard multi-armed bandits in the fixed-budget setting is provided in Theorem~\ref{theorem_lowerbound_K}.

\begin{theorem}[{Adapted from \citep[Theorem 1]{carpentier2016tight}}]
\label{theorem_lowerbound_K}
If $T \ge h^{2} \log (6 T K) /900$, then 
\begin{equation*}
    \min_{\tilde \Pi} \max_{\tilde \nu \in \tilde {\mathcal E} (h)} \Pr\left[i_{\mathrm{out}}^{\tilde \Pi} \neq 1 \right] \ge \frac 1 6 \exp \left( - \frac {240T} h \right).
\end{equation*}
Further if $h \ge 15 K^2$, then 
\begin{equation*}
    \min_{\tilde \Pi} \max_{\tilde  \nu \in \tilde {\mathcal E}(h)} \left( \Pr\left[i_{\mathrm{out}}^{\tilde \Pi} \neq 1 \right] \cdot \exp \left(  \frac {2700T} {H_{1}(\tilde  \nu) \log_2K } \right) \right)\ge \frac 1 6.
\end{equation*}
\end{theorem}

The differences between the constants of Theorem~\ref{theorem_lowerbound_K} and those of \citet{carpentier2016tight} come from the slight difference in the definitions of $H_1$. In particular, as in \citep{lattimore2020bandit, audibert2010best}, we define $H_1$ as $\sum _{1\le i\le K} \Delta_i^{-2}$ instead of $\sum _{2\le i\le K} \Delta_i^{-2}$.

Now we return to the proof of Theorem~\ref{theorem_lowerbound}.
\begin{proof}[Proof of Theorem~\ref{theorem_lowerbound}]
The idea of the proof is to reduce the linear bandit problem to the standard multi-armed bandit problem. 

We construct a special linear bandit instance $\nu$ as follows. Recall that we assume the entire set of original arm vectors $ \{a(1), a(2),\dots,a(K)\} $ span $\mathbb R ^d$, so it holds that $K\ge d$. For any arm $i\in\{1,2,\dots,d\}$, the corresponding arm vector is chosen to be $e_i$,  the $i^{\mathrm{th}}$ standard basis of $\mathbb{R}^d$. It follows that $\theta ^ * = [p(1),p(2),\cdots,p(d)]^{\top} \in \mathbb{R}^d$. For all the remaining arms $i \in \{d+1,d+2,\dots,K\}$, the corresponding arm vector $a(i)$ is chosen to be zero vector, i.e, a vector with all entries equal to $0$. Furthermore, we require the expected rewards of all the arms to be nonnegative. That is to say, $p(i) \ge 0$ for all $i \in [K]$ and in particular $p(i) = 0$ for all $i \in \{d+1,d+2,\dots,K\}$.

If the agent is given the above extra information that the expected rewards of all the arms are nonnegative (which can only help the agent improve the identification probability), then the agent knows immediately that the best arm must be among the arms $\{1,2,\dots,d\}$ since $p(1), p(2),\dots,p(K) \ge 0$. In addition, pulling the remaining arms cannot provide any useful information since the corresponding arm vectors are vectors of all zeros. Thus, the best strategy that the agent can follow is to only pull the first $d$ arms. Consequently, this linear bandit instance $\nu$ is reduced to a standard bandit instance $\tilde \nu$ with $d$ independent arms.

Therefore, Theorem~\ref{theorem_lowerbound_K} gives a minimax lower bound on the probability of misidentifying the best arm in the standard bandit instance $\tilde \nu$, due to the fact that any bounded random variable on $[0,1]$ is $1$-subgaussian. Also, following the above construction, it holds that
\begin{equation*}
    H_{1,\mathrm{lin}}(\nu) = H_{1}(\tilde \nu).
\end{equation*}
Notice that the agent cannot do better in the absence of the extra information in the linear bandit instance $\nu$. The minimax lower bound derived from Theorem~\ref{theorem_lowerbound_K} is also a minimax lower bound for the problem of best arm identification in linear bandits in the fixed-budget setting.
\end{proof}

\section{Additional implementation details and numerical results}
\label{appendix_exp}
\subsection{Additional implementation details}
\label{appendix_exp_details}
\paragraph{OD-LinBAI.} In each phase, we compute an $\epsilon$-approximate G-optimal design, where $\epsilon = 10^{-7}$. As noted in Appendix~\ref{appendix_goptimal}, this causes minimal impact on performance. Moreover, we follow the Wolfe--Atwood Algorithm with the Kumar--Yildirim start introduced in \citet{todd2016minimum}.

\paragraph{Sequential Halving.} In any linear bandit instance, we treat the $K$ arms as being independent and then apply Sequential Halving \citep{karnin2013almost}.

\paragraph{BayesGap.} For unknown parameter vector $\theta^*$, we use an uninformative prior with $\eta = 10^6$, a very large variance, for a fair comparison. In fact, through extensive tests, we notice that this parameter has limited influence on the performance. With respect to the parameter $\epsilon$ that controls the tolerance of output, although it suffices to set $\epsilon$ to be the minimum optimality gap $ \Delta_1$ theoretically, we follow the setting of \citet{hoffman2014correlation}, i.e., $\epsilon = 0$.
\begin{itemize}
    \item BayesGap-Oracle: We directly give the algorithm exact information of the required hardness quantity $H_1$.
    
    \item BayesGap-Adaptive: Following \citet{hoffman2014correlation}, we estimate the required hardness quantity by the three-sigma rule at the beginning of each time step.
\end{itemize}

\paragraph{Peace and LinearExploration.} We give an advantage to these two methods by ignoring the rounding issue and allowing fractional arm pulls, which leads to better performance. For the computation of the $\mathcal{XY}$-allocation, we follow the Frank--Wolfe heuristic algorithm in \citet{fiez2019sequential}.

\paragraph{GSE.} For the computation of the G-optimal design, we use the same method as OD-LinBAI.

\subsection{Synthetic dataset 2: random arm vectors}
\label{appendix_exp2}
In this experiment, the $K$ arm vectors are uniformly sampled from the unit $d$-dimensional sphere $\mathbb S^{d-1}$. Without loss of generality, we assume that $a(1), a(2)$ are the two closest arm vectors and then set $\theta^* = a(1) + 0.01(a(1)-a(2))$. Thus the best arm is arm $1$ while the second best arm is arm $2$. We also assume that the additive random noise follows the standard Gaussian distribution $\mathcal{N}(0, 1)$. Different from previous works \citep{tao2018best, zaki2019towards, fiez2019sequential,degenne2020gamification}, we set the number of arms to be $K = c^d$ for different integers~$c$. According to  \citet[Theorem~$8$]{cai2013distributions}, the minimum optimality gap $\Delta_1$   converges in probability to a positive number as $d$ tends to infinity so that the random linear bandit instances 
which we perform our experiments on are neither too hard nor too easy. 

\begin{figure}[htbp]
    \vskip 0.2in
	\centering
    \begin{minipage}[t]{0.75\linewidth}
		\centering
		\includegraphics[width=1\textwidth]{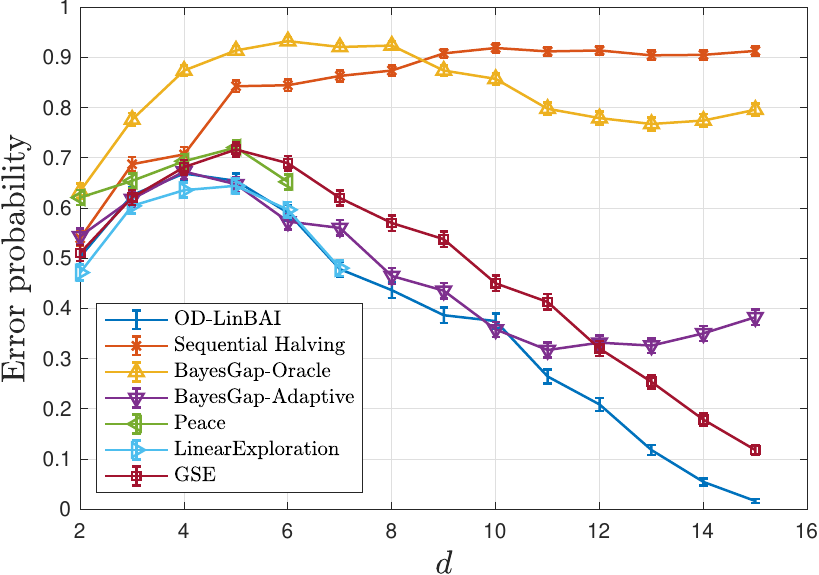}
	\end{minipage}
	\begin{minipage}[t]{0.75\linewidth}
		\centering
		\includegraphics[width=1\textwidth]{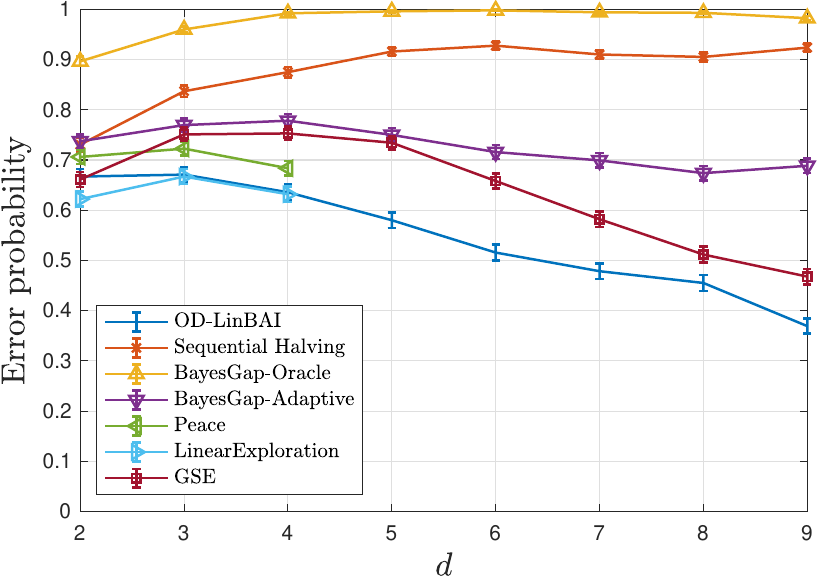}
	\end{minipage}
	\caption{Error probabilities for different $d$ with  $c=2$ at the top and $c=3$ at the bottom.}
	\label{fig_dataset2}
\end{figure}

Figure~\ref{fig_dataset2} shows the error probabilities of the $7$ different algorithms for this dataset with $c=2$ or $3$ when the time budget $T=2K$. In most situations, OD-LinBAI outperforms the other algorithms. 

It is shown in Figure~\ref{fig_dataset2} that BayesGap-Oracle does not outperform its adaptive version BayesGap-Adaptive and sometimes it even performs worse than Sequential Halving. This is partly because BayesGap-Oracle might be too conservative to converge when $T=2K$. It is noted that in UGapEb \citep{gabillon2012best}, from which BayesGap is adapted, the exploration parameter that controls how much exploration the algorithm does is tuned even if the required hardness quantity is known to the agent. Nevertheless, our algorithm OD-LinBAI is fully parameter-free.

\begin{table}[t]
  \caption{The empirical means of $\Delta_1$, $H_1$, $H_2$, $H_{1,\mathrm{lin}}$ and $H_{2,\mathrm{lin}}$ for different $d$ with $c=2$.}
  \label{table_dataset2_1}
  \vskip 0.15in
  \centering
  \begin{tabular}{llllll}
    \toprule
    $d$ &   $\Delta_1$ & $H_1$ & $H_2$ &  $H_{1,\mathrm{lin}}$ & $H_{2,\mathrm{lin}}$ \\
    \midrule
    2  & 1.040E-01 & 2.054E+09 & 2.054E+09 & 2.054E+09 & 2.054E+09 \\
    3  & 6.395E-02 & 2.498E+05 & 2.496E+05 & 2.498E+05 & 2.496E+05 \\
    4  & 4.751E-02 & 4.676E+04 & 4.666E+04 & 4.674E+04 & 4.666E+04 \\
    5  & 4.632E-02 & 5.305E+03 & 5.180E+03 & 5.264E+03 & 5.180E+03 \\
    6  & 4.360E-02 & 3.152E+03 & 2.968E+03 & 3.066E+03 & 2.968E+03 \\
    7  & 4.107E-02 & 2.427E+03 & 2.134E+03 & 2.246E+03 & 2.134E+03 \\
    8  & 3.998E-02 & 2.994E+03 & 2.490E+03 & 2.623E+03 & 2.490E+03 \\
    9  & 3.888E-02 & 2.730E+03 & 1.841E+03 & 1.988E+03 & 1.841E+03 \\
    10 & 3.848E-02 & 3.390E+03 & 1.757E+03 & 1.926E+03 & 1.757E+03 \\
    11 & 3.763E-02 & 4.809E+03 & 1.787E+03 & 1.932E+03 & 1.741E+03 \\
    12 & 3.696E-02 & 7.660E+03 & 2.453E+03 & 2.022E+03 & 1.804E+03 \\
    13 & 3.687E-02 & 1.295E+04 & 4.354E+03 & 1.924E+03 & 1.682E+03 \\
    14 & 3.657E-02 & 2.358E+04 & 8.747E+03 & 1.982E+03 & 1.708E+03 \\
    15 & 3.622E-02 & 4.434E+04 & 1.764E+04 & 2.006E+03 & 1.703E+03 \\
    \bottomrule
  \end{tabular}
\end{table}

\begin{table}[t]
  \caption{The empirical means of $\Delta_1$, $H_1$, $H_2$, $H_{1,\mathrm{lin}}$ and $H_{2,\mathrm{lin}}$ for different $d$ with $c=3$.}
  \label{table_dataset2_2}
  \vskip 0.15in
  \centering
  \begin{tabular}{llllll}
    \toprule
    $d$ &   $\Delta_1$ & $H_1$ & $H_2$ &  $H_{1,\mathrm{lin}}$ & $H_{2,\mathrm{lin}}$ \\
    \midrule
    2 & 4.896E-03 & 6.360E+14 & 6.360E+14 & 6.360E+14 & 6.360E+14 \\
    3 & 5.584E-03 & 3.767E+08 & 3.767E+08 & 3.767E+08 & 3.767E+08 \\
    4 & 5.655E-03 & 2.814E+06 & 2.812E+06 & 2.813E+06 & 2.812E+06 \\
    5 & 5.898E-03 & 3.030E+05 & 3.012E+05 & 3.023E+05 & 3.012E+05 \\
    6 & 6.070E-03 & 1.326E+05 & 1.299E+05 & 1.309E+05 & 1.299E+05 \\
    7 & 6.224E-03 & 1.128E+05 & 1.074E+05 & 1.083E+05 & 1.074E+05 \\
    8 & 6.177E-03 & 9.752E+04 & 8.441E+04 & 8.534E+04 & 8.441E+04 \\
    9 & 6.182E-03 & 1.096E+05 & 7.520E+04 & 7.621E+04 & 7.520E+04 \\
    \bottomrule
  \end{tabular}
\end{table}

The empirical means of $\Delta_1$, $H_1$, $H_2$, $H_{1,\mathrm{lin}}$ and $H_{2,\mathrm{lin}}$ for different $d$ with $c=2$ or $3$ are reported in Table~\ref{table_dataset2_1} and Table~\ref{table_dataset2_2} respectively while the empirical means of the CPU runtimes\footnote{All our experiments are implemented
in MATLAB and parallelized on an Intel(R) Core(TM) i7-4790 CPU @ 3.60GHz.} for different algorithms are listed in 
Table~\ref{table_dataset2_3} and Table~\ref{table_dataset2_4}. {The empty cells denote algorithms and instances whose complexities are too large such that their runtimes are impractical.} From these tables, we have the following observations:
\begin{enumerate}[label = (\roman*)]
    \item With the increase in the dimension of the linear bandit instances, the empirical means of the minimum optimality gap $\Delta_1$ vary a little. However, for OD-LinBAI, the linear bandit instances become easier since the time budgets grow exponentially.
    
    \item Different from synthetic dataset 1, the values of the four hardness quantities $H_1$, $H_2$, $H_{1,\mathrm{lin}}$ and $H_{2,\mathrm{lin}}$ in synthetic dataset 2 are close. This is because they are dominated by several smallest optimality gaps.
    \item OD-LinBAI shows great superiority in terms of CPU runtimes with the increase in $d$, and hence is computationally efficient compared to other methods. In particular, BayesGap is computationally intractable for synthetic dataset 2 with large $d$, due to the time-consuming matrix inverse updates at each time step. For large $d$, Peace and LinearExploration are also intractable with a reasonable computing resource, due to the time-consuming computation of the $\mathcal{XY}$-allocation via the Frank--Wolfe algorithm  heuristic algorithm in \citet{fiez2019sequential} (see Appendix \ref{appendix_exp_details}). However, the $\mathcal{XY}$-allocation does result in slightly better empirical performance as shown in the error probabilities of LinearExploration for small $d$. 
\end{enumerate}

\begin{table}
  \caption{The empirical means of the CPU runtimes for different algorithms for different $d$ with $c=2$.}
  \label{table_dataset2_3}
  \vskip 0.15in
  \centering
  \begin{tabular}{rrrrrrrrr} 
    \toprule
    &  \multicolumn{7}{c}{CPU runtimes (secs)}                   \\
    \cmidrule(r){2-8}
    $d$ & \small OD-LinBAI & \scriptsize
Sequential Halving & \scriptsize
BayesGap-Oracle & \scriptsize
BayesGap-Adaptive &\small
Peace & \scriptsize
LinearExploration & \small
GSE\\
    \midrule
    2  & 0.003 & $<$0.001 & $<$0.001  & $<$0.001 &  0.005 &  0.003 &    0.004\\
    3  & 0.006 & $<$0.001 & 0.001   & 0.001  &0.052 &0.015 &0.006\\
    4  & 0.007 & $<$0.001 & 0.001   & 0.001  &0.664 &0.129 &0.022\\
    5  & 0.015 & $<$0.001 & 0.002   & 0.003  &9.199 &1.006 &0.083\\
    6  & 0.015 & $<$0.001 & 0.006   & 0.007  &145.707 &8.527 &0.153\\
    7  & 0.009 & $<$0.001 & 0.017   & 0.022  &- &70.72 &0.222\\
    8  & 0.008 & $<$0.001 & 0.059   & 0.074  &- &- &0.397\\
    9  & 0.011 & $<$0.001 & 0.206   & 0.245  &- &- &0.378\\
    10 & 0.015 & $<$0.001 & 0.801   & 0.916  &- &- &0.504\\
    11 & 0.028 & 0.001  & 3.168   & 3.729   &- &- &0.542\\
    12 & 0.067 & 0.001  & 12.992  & 14.163  &- &- &1.119\\
    13 & 0.102 & 0.002  & 49.538  & 54.417  &- &- &2.428\\
    14 & 0.222 & 0.004  & 197.938 & 216.541 &- &- &6.538\\
    15 & 0.413 & 0.008  & 895.692 & 968.930 &- &- &30.690\\
    \bottomrule
  \end{tabular}
  \vskip -0.1in
\end{table}

\begin{table}
    \caption{The empirical means of the CPU runtimes for different algorithms for different $d$ with $c=3$}
    \label{table_dataset2_4}
    \vskip 0.15in
    \centering
  \begin{tabular}{rrrrrrrrr}
    \toprule
    &  \multicolumn{7}{c}{CPU runtimes (secs)}                   \\
    \cmidrule(r){2-8}
    $d$ & \small OD-LinBAI & \scriptsize
Sequential Halving & \scriptsize
BayesGap-Oracle & \scriptsize
BayesGap-Adaptive &\small
Peace & \scriptsize
LinearExploration & \small
GSE\\
    \midrule
      2  & 0.003 & $<$0.001 & 0.001   & 0.001 & 0.068 & 0.020 & 0.011 \\ 
      3  & 0.004 & $<$0.001 & 0.002   & 0.002 & 3.982 & 0.519 & 0.022\\
      4  & 0.004 & $<$0.001 & 0.007   & 0.010 & 311.854 & 15.078 & 0.030\\
      5  & 0.005 & $<$0.001 & 0.049   & 0.065 &- &- & 0.196 \\
      6  & 0.007 & $<$0.001 & 0.375   & 0.484 &- &- & 0.135 \\
      7  & 0.012 & 0.001  & 2.996   & 3.488   &- &- & 0.310 \\
      8  & 0.037 & 0.001  & 28.768  & 31.679  &- &- & 1.164\\
      9  & 0.107 & 0.005  & 247.645 & 280.018 &- &- & 7.707\\
      \bottomrule
    \end{tabular}
    \vskip -0.1in
 \end{table}

\subsection{Real-world dataset: Abalone dataset} 
\label{subsec:abalone}
We conduct an experiment on the Abalone dataset \citep{Dua:2019}, which  includes  $4177$ groups of $8$ attributes (such as sex, length, diameter, etc.) of the abalone as well as its target variable which is the abalone's age. The age of each abalone is usually hard to determine so it is tempting to predict the age using the $8$ attributes from physical measurements. To adapt the dataset into a linear bandit problem, we first use the whole dataset to calculate the linear regression coefficient vector $\theta^* \in \mathbb R ^ 9$ and then form a set of arm vectors by the attributes of $400$ abalones with the largest true ages. Therefore, in this real-world dataset, it holds that $d=9$ and $K=400$. We assume that the additive random noise follows a Gaussian distribution $\mathcal{N}(0, 10^2)$. The experimental results of the $5$ different algorithms\footnote{For large $d$ and $K$, the computation of the arm allocation rules in Peace and LinearExploration is intractable with a reasonable computing resource. 
See Appendix~\ref{appendix_exp2}.} are shown in Figure~\ref{fig_dataset3}. 
\begin{figure}[htbp]
	\centering
	\includegraphics[width=0.8\textwidth]{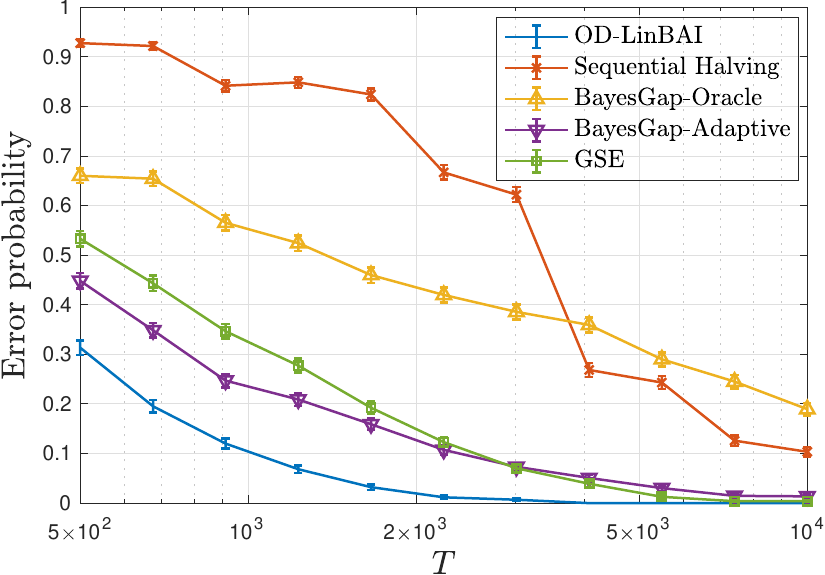}
	\caption{Error probabilities for different time budgets $T$.}
	\label{fig_dataset3}
\end{figure}

From Figure~\ref{fig_dataset3}, we see that OD-LinBAI outperforms the other competitors for all time horizons $T$.

\subsection{Comparisons to LT\&S}
Lazy Track-and-Stop (LT\&S) was proposed by \citet{jedra2020optimal} for the problem of best arm identification in linear bandits in the fixed-confidence setting, which also achieves asymptotic optimality. It is interesting to empirically investigate the fundamental difference between the fixed-confidence setting and the fixed-budget setting. We run some experiments to assess the performance of LT\&S. The experiments are based on synthetic dataset 1, which was also considered in \citet{jedra2020optimal}. Besides, in this synthetic dataset, OD-LinBAI is clearly superior to other existing methods for the fixed-budget setting (e.g., BayesGap, Peace, LinearExploration and GSE); see Section~\ref{subsec:syn1}. To adapt LT\&S to the fixed-budget setting, we omit the stopping rule, and retain the sampling rule as well as the decision rule. Besides, we consider the no averaging version of LT\&S, which demonstrates better empirical performance. The error probabilities (averaged over $8192$ independent trials) for various $K$ and $T$ are reported in Table~\ref{table_fixedconfidence}.

 \begin{table}
    \caption{Error probabilities for various $K$ and $T$.}
    \label{table_fixedconfidence}
    \vskip 0.15in
    \centering
    \begin{tabular}{lllllll}
    \toprule
    $T$ ($K=10$)  & 5     & 10    & 25    & 50    & 100   &200 \\
    \midrule
    OD-LinBAI &  0.3585  &  0.2563  &  0.1338  &  0.0610   & 0.0126  &0.0007\\
    LT\&S      &  0.3700  &  0.2629 &   0.1211  &  0.0469   &  0.0085 &0.0004\\
    \bottomrule \\[0.5cm]
    \toprule
    $T$ ($K=25$)  & 5     & 10    & 25    & 50    & 100   &200\\
    \midrule
    OD-LinBAI &   0.3640   & 0.2716  &  0.1525  &  0.0725  &0.0197 &0.0023\\
    LT\&S      &  0.4071    & 0.2722  &  0.1461  &  0.0676  &0.0178 &0.0012\\
    \bottomrule \\[0.5cm]
    \toprule
    $T$ ($K=35$)  & 10    & 25    & 50    & 100    &200 & 300\\
    \midrule
    OD-LinBAI &     0.2777   & 0.1615    &0.0719    & 0.0204 &0.0020& 0.0005\\
    LT\&S      &  0.2883    &0.1581  &  0.0823  & 0.0245 &0.0027  &0.0002\\
    \bottomrule \\[0.5cm]
    \toprule
    $T$ ($K=50$)  & 10    & 25    & 50    & 100    &200 & 300 \\
    \midrule
    OD-LinBAI &    0.2861    &0.1617   & 0.0778    & 0.0232  & 0.0023 & 0.0005\\
    LT\&S      &    0.2894    &0.1731   & 0.0924   & 0.0317  & 0.0045 &0.0010\\
    \bottomrule 
    \end{tabular}
    \vskip -0.1in
 \end{table}

 OD-LinBAI outperforms LT\&S when $T$ is small and is inferior to LT\&S when $T$ is large enough.\footnote{This phenomenon is not observed when $K = 50$ since $T = 300$ is not large enough for this observation to be made. However, we note that even with $8192$ independent trials, this is not sufficient to estimate the minimal error probabilities with high enough statistical confidence.} The results are consistent with what we expect based on the sampling rule of LT\&S since it may perform  sub-optimally   when the number of  time steps $T$  is small  but LT\&S is guaranteed to converge to the optimal rule  as $T$ tends to infinity. However, it remains open as to whether this greedy method is close-to-optimal in the fixed-budget setting.  In contrast, OD-LinBAI is minimax optimal in the fixed-budget regime; see Section~\ref{subsection_lower}. Thus, these two algorithms work well in different regimes. 


\end{document}